\documentclass[sn-chicago]{sn-jnl}
\usepackage[utf8]{inputenc}
\usepackage{natbib}
\usepackage{graphicx}
\usepackage{amssymb}
\usepackage[normalem]{ulem}
\usepackage{amsmath}
\usepackage{amsthm}
\usepackage{xcolor}
\usepackage{placeins}
\usepackage{caption}
\usepackage{subcaption}
\usepackage{hyperref}

\jyear{2021}

\theoremstyle{thmstyleone}
\newtheorem{theorem}{Theorem}

\theoremstyle{thmstyletwo}%
\newtheorem{remark}{Remark}

\begin{document}

\title[Constrained Policy Gradient Method for Safe and Fast Reinforcement Learning]{Constrained Policy Gradient Method for Safe and Fast Reinforcement Learning: a Neural Tangent Kernel Based Approach}

\author*[1]{\fnm{Balázs} \sur{Varga} \email{balazsv@chalmers.se}}

\author[1]{\fnm{Balázs} \sur{Kulcsár} \email{kulcsar@chalmers.se}}

\author[2]{\fnm{Morteza Haghir} \sur{Chehreghani} \email{morteza.chehreghani@chalmers.se}}

\affil*[1]{\orgdiv{Department of Electrical Engineering}, \orgname{Chalmers University of Technology}, \orgaddress{\street{Hörsalsvägen 11}, \city{Gothenburg}, \country{Sweden}}}

\affil[2]{\orgdiv{Department of Computer Science and Engineering}, \orgname{Chalmers University of Technology}, \orgaddress{\street{Hörsalsvägen 11}, \city{Gothenburg}, \country{Sweden}}}

\abstract{This paper presents a constrained policy gradient algorithm. We introduce constraints for safe learning with the following steps. First, learning is slowed down (lazy learning) so that the episodic policy change can be computed with the help of the policy gradient theorem and the neural tangent kernel. Then, this enables us the evaluation of the policy at arbitrary states too. In the same spirit, learning can be guided, ensuring safety via augmenting episode batches with states where the desired action probabilities are prescribed. Finally, exogenous discounted sum of future rewards (returns) can be computed at these specific state-action pairs such that the policy network satisfies constraints. Computing the returns is based on solving a system of linear equations (equality constraints) or a constrained quadratic program (inequality constraints, regional constraints). 
Simulation results suggest that adding constraints (external information) to the learning can improve learning in terms of speed and transparency reasonably if constraints are appropriately selected. 
The efficiency of the constrained learning was demonstrated with a shallow and wide ReLU network in the Cartpole and Lunar Lander OpenAI gym environments.
The main novelty of the paper is giving a practical use of the neural tangent kernel in reinforcement learning.}

\maketitle

\section{Introduction}
In reinforcement learning (RL), the agent learns in a trial and error way. 
In a real setting, it can lead to undesirable situations which may result in damage or injury of the agent or the environment system. In addition, the agent might waste a significant amount of time exploring irrelevant regions of the state and action spaces. 
Safe RL can be defined as the process of learning policies that maximize the expectation of the return in problems under safety constraints. Thus, safe exploration often includes some prior knowledge of the environment (e.g.,~a model \cite{berkenkamp2017safe, fisac2018general, zimmer2018safe}) or has a risk metric \cite{garcia2015comprehensive, turchetta2020safe}.
In RL, safety can be guaranteed in a probabilistic way. Learning in this context aims to strike a balance between exploration and exploitation so that the system remains within a safe set. On the other hand, in a safety-critical setting, exploration cannot be done blindly. Therefore, some sort of knowledge of the environment is vital. 

Constrained learning is an intensively studied topic, having close ties to safe learning \cite{yang2019advancing}. 
\cite{han2008modified} argues that constrained learning has better generalization performance and a faster convergence rate. \cite{tessler2018reward} presents a constrained policy optimization, which uses an alternative penalty signal to guide the policy. \cite{uchibe2007constrained} proposes an actor-critic method with constraints that define the feasible policy space.

In policy gradient methods, the function approximator predicts action values or probabilities from which the policy can be derived directly. Finding the parameters of the agent is done via optimization (gradient ascent) following the policy gradient theorem \cite{sutton2018reinforcement}. Policy gradient methods have many variants and extensions to improve their learning performance \cite{zhang2021multi}. For example, in \cite{kakade2001natural}  a covariant gradient is defined based on the underlying structure of the policy. \cite{ciosek2018expected} deduced expected policy gradients for actor-critic methods. \cite{cheng2020trajectory} deals with the reduction of the variance in Monte-Carlo (MC) policy gradient methods.

In this work we use priori knowledge of the environment not to build a model, rather than to impose some constraints on the policy.  We develop a deterministic policy gradient algorithm (based on REINFORCE, \cite{williams1987class}) augmented with different types of constraints via shaping the rewards. Opposed to \cite{altman2019constrained}, transition probabilities are not explicitly replaced with taboo states but constraints are imposed on these transition probabilities via solving a constrained quadratic program. Therefore, compared to \cite{tessler2018reward}, training does not rely on non-convex optimization or additional heuristics. Safety in our approach is guaranteed by the careful selection of constraints and the success of the optimization. 

The key ingredient in our work is the Neural Tangent Kernel (NTK, \cite{jacot2018neural}).
"It the evolution of fully connected neural networks under gradient descent in function space. Dual to this perspective is an understanding of how neural networks evolve in parameter space, since the NTK is defined in terms of the gradient of the NN's outputs with respect to its parameters. In the infinite width limit, the connection between these two perspectives becomes especially interesting. The NTK remaining constant throughout training at large widths co-occurs with the NN being well described throughout training by its first order Taylor expansion around its parameters at initialization" \cite{sohl2020infinite}. 
In this work, we exploit the NTK to compute the future policy using the above properties (fully connected, very wide NN, slow learning rate).

Earlier, policy iteration has been used in conjunction with NTK for learning the value function \cite{goumiri2020reinforcement}. On the other hand, it was used in its analytical form as the covariance kernel of a GP (\cite{novak2019neural, rasmussen2003gaussian, yang2019fine}). Here, we directly use the NTK to project a one-step policy change in conjunction with the policy gradient theorem.
Then, safety is incorporated via constraints. It is assumed that there are states where the agent's desired behaviour is known. At these "safe" states action probabilities are prescribed as constraints for the policy. Finally, returns are computed in such a way that the agent satisfies the constraints. This assumption is mild when considering physical systems: limits of a controlled system (saturation) or desired behavior at certain states are usually known (i.e.,~the environment is considered a gray-box). The proposed algorithm is developed for continuous state spaces and discrete action spaces. According to the proposed categorization in \cite{garcia2015comprehensive}, the proposed algorithm falls into constrained optimization with external knowledge. 

The contribution of the paper is twofold. First, we analytically give the policy evolution under gradient flow, using the NTK. Second, we extend the REINFORCE algorithm with constraints. Our variant of the REINFORCE algorithm converges within a few episodes if constraints are set up correctly. The constrained extension relies on computing extra returns via convex optimization. In conclusion, the paper provides a practical use of the neural tangent kernel in reinforcement learning.

The paper is organized as follows. First, we present the episode-by-episode policy change of the REINFORCE algorithm (Section \ref{sec:REINFORCE}). Then, relying on the NTK, we deduce the policy change at unvisited states, see Section \ref{sec:anypoint}. Using the results in Section \ref{sec:anypoint}, we compute returns at arbitrary states in Section \ref{sec:constr}. We introduce equality constraints for the policy by computing "safe" returns by solving a system of linear equations (Section \ref{sec:eq_constr}). In the same context, we can enforce inequality constraints by solving a constrained quadratic program, see Section \ref{sec:ineq_constr}. Additionally, we explore various ways to enforce constraints on a whole region (Section \ref{sec:regional_constraints}). In Section \ref{sec:cartpole} we investigate the proposed learning algorithm in two OpenAI gym environments: in the Cartpole environment and in the Lunar Lander (Section \ref{sec:lunarlander}). Finally, we summarize the findings of this paper in Section \ref{sec:concl}.

\section{Kernel-based analysis of the REINFORCE algorithm}
\label{sec:learning}
In this section, the episode-by-episode learning dynamics of a policy network is analyzed and controlled in a constrained way. To this end, first we introduce the RL framework and the REINFORCE algorithm. Then, the learning dynamics of a wide and shallow neural network is analytically given. Finally, returns are calculated that force the policy to obey equality and inequality constraints at specific states.

\subsection{Reformulating the learning dynamics of the REINFORCE algorithm}
\label{sec:REINFORCE}

Reinforcement learning problems are commonly introduced in a Markov Decision Process (MDP) setting \cite{sutton2018reinforcement}. Similarly, the most common way of tackling safe RL is through constrained MDPs. I.e.,~safety is ensured via constraining the MDP: at given states some actions that are deemed unsafe are forbidden \cite{altman1999constrained, wachi2020safe}. 

Let the 5-tuple $(\mathcal{S}, \mathcal{A}, \underline{\underline{P}}, \mathcal R(s), \gamma)$ characterize an MDP. This tuple consists of the continuous state space with $n_n$ dimensions $\mathcal{S} \subseteq \mathbb R^{n_n}$, the discrete action space $\mathcal{A} \subset \mathbb{Z}^{n_a}$, the transition probability matrix $\underline{\underline{P}}$, the reward function $\mathcal R(s) \in \mathbb{R}$, and the discount factor $\gamma \in ]0,1]$. The agent traverses the MDP following the policy $\pi(a \mid s, \theta)$.

In the reinforcement learning setup, the goal is maximizing the (discounted sum of future) rewards in episode $e$
\begin{equation}
    J(\theta) = \mathbb{E}_\pi\left(G_e\right).
\end{equation}
Then, policy gradient methods learn by applying the policy gradient theorem.
\begin{theorem}{\textbf{Policy Gradient Theorem. \cite{sutton2018reinforcement}}}
The derivative of the expected reward is the expectation of the product of the reward and gradient of the log of the policy 
\begin{equation}
\label{eq:pg_theorem}
    \frac{\partial }{\partial \theta}  \mathbb{E}_\pi\left(G_e\right) =  \mathbb{E}_\pi\left(G_e \frac{\partial }{\partial \theta} log  \pi(a \mid s, \theta)\right).
\end{equation}
\end{theorem}

Learning means tuning the weights of a function approximator (agent) episode-by-episode. In most cases the function approximator is a Neural Network \cite{arulkumaran2017deep, silver2014deterministic}. If the training is successful, then the function approximator predicts values or probabilities from which the policy can be derived. More specifically, the output of the policy gradient algorithm is a probability distribution, which is characterized by the function approximator’s weights. 

In this paper we deal with the a simple policy gradient methods called the REINFORCE algorithm \cite{williams1987class, szepesvari2010algorithms}.
In REINFORCE, the agent learns the policy directly by updating its weights $\theta(e)$ using Monte-Carlo episode samples. Thus, the expectation in Eq.~\eqref{eq:pg_theorem} turns into summation. One realization is generated with the current policy (at episode $e$) $\pi(a_e \mid s_e, \theta(e))$. Assuming continuous state space, and discrete action space the episode batch (with length  $n_B$) is $\{\underline{s}_e(k), \underline{a}_e(k), \underline{r}_e(k)\}$, for all $k = 1,2, ..., n_{B}$, where $\underline{s}_e(k) \in \mathcal{S}$ is the $n_n$ dimensional state vector in the $k^{th}$ step of the MC trajectory, $\underline{a}_e(k) \in \mathcal{A}$ is the action taken, and $\underline{r}_e(k) \in \mathcal R(s)$ is the reward in step $k$. For convenience, the states, actions, and rewards in batch $e$ are organized into columns $\underline{s}_e \in \mathbb{R}^{n_B \times n_n}$, $\underline{a}_e \in \mathbb{Z}^{n_B}$, $\underline{r}_e \in \mathbb{R}^{n_B}$, respectively. The REINFORCE algorithm learns as summarized in Algorithm \ref{alg:REINFORCE}. The update rule is based on the policy gradient theorem \cite{sutton2018reinforcement} and for the whole episode it can be written as the sum of gradients induced by the batch:
\begin{equation}
\label{eq:gradientascent}
\theta(e+1)\! =\! \theta(e) + \alpha  \!\sum_{k = 1}^{n_{B}}\!\! \left( \! G_e(k) \frac{\partial}{\partial \theta} log \pi(\underline{a}_e(k)\mid\underline{s}_e(k),\theta(e))\! \right)^T
\end{equation}

\begin{algorithm}[htb!]
\caption{The REINFORCE algorithm}
\begin{algorithmic}[1]
\State Initialize $e = 1$.
\State Initialize the policy network with random $\theta$ weights.
\While{not converged}
\State {Generate a Monte-Carlo trajectory $\{\underline{s}_e(k), \underline{a}_e(k), \underline{r}_e(k)\}$, 

$k = 1,2, ..., n_{B}$ with the current policy $\pi(\underline{a}_e(k) \mid \underline{s}_e(k), \theta(e))$. }
\For{the whole MC trajectory ($k=1,2,...,n_{B}$)} \State Estimate the return as 
        \begin{equation}
        \label{eq:return}
            G_e(k) = \sum_{\kappa =k+1}^{n_{B}} \gamma^{\kappa-k}\underline{r}_e(\kappa) \in \mathbb{R}
        \end{equation} 
        \State Update policy parameters with gradient ascent according to Eq.~\eqref{eq:gradientascent}.
        \EndFor
\State Increment $e$.
\EndWhile
\end{algorithmic}
\label{alg:REINFORCE}
\end{algorithm}

Taking Algorithm \ref{alg:REINFORCE} further, it is possible to compute the episodic policy change $\frac{\partial\underline \pi(\underline a_e \mid \underline s_e, \theta(e))}{\partial e}$ with respect to every element in batch $e$, assuming gradient flow ($\alpha\rightarrow 0$). Note that $\underline \pi(\underline a_e \mid \underline s_e, \theta(e))$ is now vector-valued since we are dealing with the whole batch.

\begin{theorem}
\label{theo:learningdynamics}
Given batch $\{\underline{s}_e(k), \underline{a}_e(k), \underline{r}_e(k)\}_{k=1}^{n_B}$, and assuming gradient flow, the episodic policy change with the REINFORCE algorithm at the batch state-action pairs are 
\begin{equation*}
\frac{\partial\underline \pi(\underline a_e \mid \underline s_e, \theta(e))}{\partial e} =  \underline {\underline {\Theta}}_{\pi,e}(s,s') \underline{\underline{\Pi}}^I_e(\underline{s}_e, \underline{a}_e, \theta(e))\underline G_e(\underline r_e),
\end{equation*}
where $\underline {\underline {\Theta}}_{\pi,e}(s,s') \in \mathbb{R}^{n_B \times n_B}$ is the neural tangent kernel, $\underline{\underline{\Pi}}^I_e(\underline{s}_e, \underline{a}_e, \theta(e)) \in \mathbb{R}^{n_B \times n_B}$ is a diagonal matrix containing the inverse policies (if they exist) at state-action pairs of batch $e$, and $\underline G_e(\underline r_e) \in \mathbb{R}^{n_B}$ is the vector of returns.
\end{theorem}

\begin{proof}
Assuming very small learning rate $\alpha$, the update algorithm (gradient ascent) can be written in continuous form (gradient flow) \cite{parikh2014proximal}:
\begin{equation}
\frac{d \theta(e)}{de}= \sum_{k = 1}^{n_{B}}\left(G_e(k) \frac{\partial}{\partial \theta} log \pi(\underline{a}_e(k)\mid\underline{s}_e(k),\theta_{e})\right)^T. 
\end{equation}
Furthermore, to avoid division by zero, it is assumed that the evaluated policy is strictly positive. The derivative on the left hand side is a column vector with size $n_\theta$. Denote $\underline s_e = \{\underline{s}_e(k)\}_{k=1}^{n_{B}}$, $\underline a_e = \{\underline{a}_e(k)\}_{k=1}^{n_{B}}$, $\underline r_e = \{\underline{r}_e(k)\}_{k=1}^{n_{B}}$ and rewrite the differential equation in vector form as 
\begin{equation}
\label{eq:PG_vector}
    \frac{d \theta(e)}{de} =  \underline{ \underline{\dot \Pi}}_{log}(\underline s_e, \underline a_e, \theta(e)) \underline G_e(\underline r_e),
\end{equation}
where 
\begin{equation}
   \underline{G}_e(\underline r_e)  = [G_e(1), G_e(2), ..., G_e(n_{B})]^T
\end{equation}
is the vector of returns in episode $e$ (Eq.~\eqref{eq:return}). The matrix of partial log policy derivatives ($\underline{ \underline{\dot \Pi}}_{log}(\underline s_e, \underline a_e, \theta(e)) \in \mathbb{R}^{n_\theta\times n_{B}}$) is
\begin{align}
\nonumber
& \underline{ \underline{\dot \Pi}}_{log}(\underline s_e, \underline a_e, \theta(e)) = \\ 
& \begin{bmatrix}
\frac{\partial log \pi(\underline{a}_e(1)\mid\underline{s}_e(1), \theta(e))}{\partial \theta_1} & 
\! \dots \! & \frac{\partial log \pi(\underline{a}_e({n_{B}})\mid\underline{s}_e({n_{B}}), \theta(e))}{\partial \theta_1} \\ 
\frac{\partial log \pi(\underline{a}_e(1)\mid\underline{s}_e(1), \theta(e))}{\partial \theta_2} & 
\! \dots \! & \frac{\partial log \pi(\underline{a}_e({n_{B}})\mid\underline{s}_e({n_{B}}), \theta(e))}{\partial \theta_2} \\ 
\vdots &
\! \ddots \!  & \vdots\\ 
\frac{\partial log \pi(\underline{a}_e(1)\mid\underline{s}_e(1), \theta(e))}{\partial \theta_{n_\theta}} & 
\! \dots \! & \frac{\partial log \pi(\underline{a}_e({n_{B}})\mid\underline{s}_e({n_{B}}), \theta(e))}{\partial \theta_{n_\theta}} 
\end{bmatrix},
\end{align}
where subscripts of $\theta$ denote weights and biases of the policy network.
By using an element-wise transformation $\frac{logf(t)}{dx}=\frac{f'(t)}{f(t)}$,  $\underline{ \underline{\dot \Pi}}_{log}(\underline s_e, \underline a_e, \theta(e))$ can be rewritten as a product:
\begin{align}
\nonumber
   & \underline{ \underline{\dot \Pi}}_{log}(\underline s_e, \underline a_e, \theta(e)) =  \\ 
& \begin{bmatrix}
\frac{\partial \pi(\underline{a}_e(1)\mid\underline{s}_e(1), \theta(e))}{\partial \theta_1} & 
\! \dots \!  & \frac{\partial \pi(\underline{a}_e({n_{B}})\mid\underline{s}_e({n_{B}}), \theta(e))}{\partial \theta_1} \\ 
\frac{\partial \pi(\underline{a}_e(1)\mid\underline{s}_e(1), \theta(e))}{\partial \theta_2} & 
\! \dots \! & \frac{\partial \pi(\underline{a}_e({n_{B}})\mid\underline{s}_e({n_{B}}), \theta(e))}{\partial \theta_2} \\ 
\vdots & \! \ddots \! & \vdots\\ 
\frac{\partial \pi(\underline{a}_e(1)\mid\underline{s}_e(1), \theta(e))}{\partial \theta_{n_\theta}} & 
\! \dots \! & \frac{\partial \pi(\underline{a}_e({n_{B}})\mid\underline{s}_e({n_{B}}), \theta(e))}{\partial \theta_{n_\theta}} 
\end{bmatrix} \cdot \nonumber \\
& \cdot \begin{bmatrix}
\frac{1}{\pi(\underline{a}_e(1)\mid\underline{s}_e(1), \theta(e))} & \dots  & 0\\ 
\vdots & \ddots  & \vdots\\ 
0 & \dots  & \frac{1}{\pi(\underline{a}_e({n_{B}})\mid\underline{s}_e({n_{B}}), \theta(e))}
\end{bmatrix},
\label{eq:pilog2}
\end{align}
where the matrix on the left is a transposed Jacobian, i.e.,~$\underline{ \underline{\dot \Pi}}(\underline s_e, \underline a_e, \theta(e)) = \left(\frac{\partial}{\partial \theta} \underline \pi(\underline a_e \mid \underline s_e, \theta(e))\right)^T$. Denote the inverse policies $\frac{1}{\pi(\underline{a}_e(k)\mid\underline{s}_e(k), \theta(e))}$ with $\pi^I_e(k)$ and $\underline \pi^I_e(\underline s_e, \underline a_e, \theta(e)) = \{\pi^I_e(k)\}_{k=1}^{n_{B}} \in \mathbb{R}^{n_B}$. The change of the agent weights based on batch $e$ is
\begin{equation}
\label{eq:theta_nonlin}
    \frac{d \theta(e)}{de }  =  
    \underline{ \underline{\dot \Pi}}(\underline s_e, \underline a_e, \theta(e))  \underline {\underline {\Gamma} }_e(\underline r_e) \underline \pi^I_e(\underline s_e, \underline a_e, \theta(e))
\end{equation}
with $\underline {\underline {\Gamma} }_e(\underline r_e)= diag(\underline G_e(\underline r_e) ) \in \mathbb{R}^{n_B \times n_B}$.
Next, write the learning dynamics of the policy, using the chain rule:
\begin{equation}
    \frac{\partial\underline \pi(\underline a_e \mid\underline s_e, \theta(e))}{\partial e} = \frac{\partial}{\partial \theta} \underline \pi(\underline a_e \mid\underline s_e,\theta(e))   \frac{d\theta(e)}{de}.
\end{equation}
First, extract $\frac{d\theta(e)}{de}$ as in Eq.~\eqref{eq:theta_nonlin}:
\begin{align}
\nonumber
    & \frac{\partial\underline \pi(\underline a_e \mid\underline s_e,\theta(e))}{\partial e} =  \\ 
    &\underline{ \underline{\dot \Pi}}(\underline s_e, \underline a_e, \theta(e))^T \underline{ \underline{\dot \Pi}}(\underline s_e, \underline a_e, \theta(e)) \underline {\underline {\Gamma} }_e(\underline r_e) \underline \pi^I_e(\underline s_e, \underline a_e, \theta(e)).
\end{align}
Note that $\underline{ \underline{\dot \Pi}}(\underline s_e, \underline a_e, \theta(e))^T \underline{ \underline{\dot \Pi}}(\underline s_e, \underline a_e, \theta(e)) = \frac{\partial}{\partial \theta} \underline \pi(\underline a_e \mid \underline s_e, \theta(e)) \left(\frac{\partial}{\partial \theta} \underline \pi(\underline a_e \mid \underline s_e, \theta(e))\right)^T$ is the Neural Tangent Kernel (NTK), as defined in \cite{jacot2018neural}. Denote it with ${\underline {\underline \Theta}}_{\pi,e}(s,s') \in \mathbb{R}^{n_{B} \times n_{B}}$. Finally, the policy update due to episode batch $e$ at states $\underline{s}_e(k)$ for actions $\underline{a}_e(k)$ become:
\begin{equation}
\label{eq:pi_nonlin_0}
    \frac{\partial\underline \pi(\underline a_e \mid \underline s_e, \theta(e))}{\partial e} =  \underline {\underline {\Theta}}_{\pi,e}(s,s') \underline {\underline {\Gamma} }_e \underline \pi^I_e(\underline s_e, \underline a_e, \theta(e)).
\end{equation}
Similarly, by defining $\underline{\underline{\Pi}}^I_e(\underline{s}_e, \underline{a}_e, \theta(e)) = diag(\underline \pi^I_e(\underline s_e, \underline a_e, \theta(e)))$ we can write Eq.~\eqref{eq:pi_nonlin} as
\begin{equation}
\label{eq:pi_nonlin}
    \frac{\partial\underline \pi(\underline a_e \mid \underline s_e, \theta(e))}{\partial e} =  \underline {\underline {\Theta}}_{\pi,e}(s,s') \underline{\underline{\Pi}}^I_e(\underline{s}_e, \underline{a}_e, \theta(e))\underline G_e(\underline r_e)
\end{equation}
too. 
\end{proof}
\begin{remark}
From Eq.~\eqref{eq:pi_nonlin_0}, the learning dynamics of the REINFORCE algorithm is a nonlinear differential equation system. If the same data batch is shown to the neural network over and over again, the policy evolves as $\dot{x}(\tau) = \beta \frac{1}{x(\tau)}$. 
\end{remark}

\subsection{Evaluating the policy change for arbitrary states and actions}
\label{sec:anypoint}
In this section, we describe the policy change for any state and any action if the agent learns from batch $e$. In most cases, the learning agent can perform multiple actions. Assume the agent can take $a^1, a^2, ..., a^{n_A}$ actions (the policy net has $n_A$ output channels). Previous works that deal with NTK all consider one output channel in the examples (e.g.,~\cite{jax2018github, jacot2018neural, yang2019fine}), however state that it works for multiple outputs too. \cite{jacot2018neural} claim that a network with $n_A$ outputs can be handled as $n_A$ independent neural networks. On the other hand, it is possible to fit every output channel into one equation by generalizing Theorem \ref{theo:learningdynamics}.
First, reconstruct the return vector $\underline G_e$ as follows.
\begin{equation}
\underline G_e(\underline a_e, \underline r_e ) = [G_e^{a_1}(1), G_e^{a_2}(1), ...,  G_e^{a_{n_A}}(n_{B})]^T \in \mathbb{R}^{n_An_{B}},
\end{equation}
where 
\begin{equation}
\label{eq:G_cases}
G_e^{a_i}(k) = 
\begin{cases}
G_e(k) & \text{if } a^i \text{ is taken at step } k \\
0 & \text{otherwise}
\end{cases}
\end{equation}
and $i = 1,2,...n_A$. In other words, $\underline G_e(\underline a_e, \underline r_e )$ consists of $n_A \times 1$ sized blocks with zeros at action indexes which are not taken at $\underline s_e(k)$, ($n_A-1$ zeros) and the original return (Eq.~\eqref{eq:return}) at the position of the taken action. Note that action dependency of $\underline G_e(\underline a_e, \underline r_e )$ stems from this assumption.

\begin{theorem}
\label{theo:learningdynamics_multi}
Given batch $\{\underline{s}_e(k), \underline{a}_e(k), \underline{r}_e(k)\}_{k=1}^{n_B}$, and assuming gradient flow, the episodic policy change with the REINFORCE algorithm at the batch states for an $n_A$ output policy network is  
\begin{equation*}
    \frac{\partial\underline \pi(\underline a \mid \underline s_e, \theta(e))}{\partial e} =  \underline{\underline{\Theta}}_{\pi,e}(s,s')\underline{\underline{\Pi}}^I_e(\underline{s}_e, \theta(e)) \underline G_e(\underline a_e, \underline r_e),
\end{equation*}
where $\frac{\partial\underline \pi(\underline a \mid \underline s_e, \theta(e))}{\partial e} \in\mathbb{R}^{n_An_{B}}$ and $\underline{\underline{\Theta}}_{\pi,e}(s,s') \in\mathbb{R}^{n_An_{B}\times n_An_{B}}$. 
\end{theorem}
\begin{proof}
We can deduce the multi-output case by rewriting Eq.~\eqref{eq:PG_vector}. For simplicity, we keep the notations, but the matrix and vector sizes are redefined for the vector output case. Therefore, the log policy derivatives are evaluated for every possible action at the states contained in a batch: 
\begin{align}
\nonumber
    &\underline{\underline{\dot \Pi}}_{log}(\underline{s}_e, \theta(e)) = \\
    &\begin{bmatrix}
\frac{\partial log \pi(a^{1}\mid\underline{s}_e(1), \theta(e))}{\partial \theta_1} &
\! \dots \! & \frac{\partial \pi({a}^{n_A}\mid\underline{s}_e({n_{B}}), \theta(e))}{\partial \theta_1} log \\
\frac{\partial }{\partial \theta_2} log \pi(a^{1}\mid\underline{s}_e(1), \theta(e))&
\! \dots \!  & \frac{\partial log \pi({a}^{n_A}\mid\underline{s}_e({n_{B}}), \theta(e))}{\partial \theta_2} \\
\vdots  & \! \ddots \! & \vdots\\ 
\frac{\partial log \pi(a^{1}\mid\underline{s}_e(1), \theta(e))}{\partial \theta_{n_\theta}} &
\!\dots \! & \frac{\partial log \pi({a}^{n_A}\mid\underline{s}_e({n_{B}}), \theta(e))}{\partial \theta_{n_\theta}} \\
\end{bmatrix},
\end{align}
$\underline{\underline{\dot \Pi}}_{log}(\underline{s}_e, \theta(e)) \in \mathbb{R}^{n_\theta \times n_An_{B}}$. The zero elements in $\underline G_e(\underline a_e, \underline r_e)$ will cancel out log probabilities of actions that are not taken in episode $e$, see Eq.~\eqref{eq:G_cases}. Therefore, the final output $\frac{\partial\underline \pi(\underline a_e \mid \underline s_e, \theta(e))}{\partial e}$ does not change. Note that, the action dependency is moved from $\underline{\underline{\dot \Pi}}_{log}(\underline{s}_e, \theta(e))$ to $\underline G_e(\underline a_e, \underline r_e )$. That is, because the policy is evaluated for every output channel of the policy network, but nonzero reward is given only if an action is actually taken in the MC trajectory. Continue by separating $\underline{\underline{\dot \Pi}}_{log}(\underline{s}_e, \theta(e))$ into two matrices, in the same way as in Eq.~\eqref{eq:pilog2}.
\begin{align}
\nonumber
& \underline{\underline{\dot \Pi}}_{log}(\underline{s}_e, \theta(e)) =  \underline{\underline{\dot \Pi}}(\underline{s}_e, \theta(e))\underline{\underline{\Pi}}^I_e(\underline{s}_e, \theta(e)) = \\ \nonumber
 & \!\!\! \begin{bmatrix}
    \frac{\partial \pi(a^{1}\mid\underline{s}_e(1), \theta(e))}{\partial \theta_1} \!\! & \!\!
    \frac{\partial \pi(a^{2}\mid\underline{s}_e(1), \theta(e))}{\partial \theta_1}  &
    \!\!\! \dots \!\!\! & \frac{\partial \pi({a}^{n_A}\mid\underline{s}_e({n_{B}}), \theta(e))}{\partial \theta_1} \\
    \frac{\partial \pi(a^{1}\mid\underline{s}_e(1), \theta(e))}{\partial \theta_2} \!\! & \!\!
    \frac{\partial \pi(a^{2}\mid\underline{s}_e(1), \theta(e))}{\partial \theta_2}  &
    \!\!\! \dots \!\!\! & \frac{\partial \pi({a}^{n_A}\mid\underline{s}_e({n_{B}}), \theta(e))}{\partial \theta_2} \\
    \vdots &  \vdots & \!\!\!\ddots \!\!\! & \vdots\\ 
    \frac{\partial \pi(a^{1}\mid\underline{s}_e(1), \theta(e))}{\partial \theta_{n_\theta}} \!\! & \!\!
    \frac{\partial \pi(a^{2}\mid\underline{s}_e(1), \theta(e))}{\partial \theta_{n_\theta}}  &
    \!\!\! \dots \!\!\! & \frac{\partial \pi({a}^{n_A}\mid\underline{s}_e({n_{B}}), \theta(e))}{\partial \theta_{n_\theta}} 
    \end{bmatrix} \! \! \cdot \\
&\cdot \! \begin{bmatrix}
\frac{1 }{ \pi(a^{1}\mid\underline{s}_e(1), \theta(e))} \!& & & \\ 
& \! \frac{1 }{\pi({a}^{2}\mid\underline{s}_e({1})\theta(e))} \! & &\\ 
& & \! \ddots \!& \\ 
& & & \! \frac{1 }{\pi({a}^{n_A}\mid\underline{s}_e({n_{B}}), \theta(e))}
\end{bmatrix}\!,
\end{align}
where the diagonalized inverse policies are denoted with with $\underline{\underline{\Pi}}^I_e(\underline{s}_e, \theta(e))$. The weight change can be written as
\begin{equation}
    \frac{d \theta(e)}{de }  = \underline{\underline{\dot \Pi}}(\underline{s}_e, \theta(e))\underline{\underline{\Pi}}^I_e(\underline{s}_e, \theta(e)) \underline G_e(\underline a_e, \underline r_e).
\end{equation}
Following the same steps as for the proof of Theorem \ref{theo:learningdynamics}, the policy change for every output channel is 
\begin{equation}
\label{eq:pi_nonlin_multi}
    \frac{\partial\underline \pi(\underline a \mid \underline s_e, \theta(e))}{\partial e} =  \underline{\underline{\Theta}}_{\pi,e}(s,s')\underline{\underline{\Pi}}^I_e(\underline{s}_e, \theta(e)) \underline G_e(\underline a_e, \underline r_e).
\end{equation}
\end{proof}

\begin{remark}
It is possible to write the average change of each policy output channel over a batch (with superscript $B$) as 
\begin{align}
\label{eq:batch_change}
\nonumber &
    \frac{\partial\underline \pi^B(\underline a \mid \underline s_e, \theta(e))}{\partial e} = \\ & \frac{1}{n_{B}}\underline{\underline{E}}\;\underline{\underline{\Theta}}_{\pi,e}(s,s')\underline{\underline{\Pi}}^I_e(\underline{s}_e, \theta(e)) \underline G_e(\underline a_e, \underline r_e).
\end{align}
where $\underline{\underline{E}}$ is an $n_A \times n_An_{B}$ matrix consisting of $n_{B}$ eye matrices of size $n_A \times n_A$. $\underline{\underline{E}}$ is used to sum elements that correspond to the same output channel. The result of Eq.~\eqref{eq:batch_change} is the policy change at an averaged state of $\underline s_e$. 
\end{remark} 
With Theorem \ref{theo:learningdynamics_multi}, it is possible to evaluate how the policy will change at states $\underline s_e$ when learning on batch data $e$ if the learning rate is small.
By manipulating Eq.~\eqref{eq:pi_nonlin_multi} policy change can be evaluated at states not part of batch $\underline s_e$ too.
\begin{theorem}
\label{theo:learningdynamics_anywhere}
Given batch $\{\underline{s}_e(k), \underline{a}_e(k), \underline{r}_e(k)\}_{k=1}^{n_B}$, and assuming gradient flow, the episodic policy change with the REINFORCE algorithm at any state $s_0 \notin \underline s_e$ is 
\begin{equation*}
\frac{\partial\underline \pi(a \mid s_0, \theta(e))}{\partial e} =  \underline{\underline{\vartheta}}(\underline s_e,s_0) \underline{\underline{\Pi}}^I_e(\underline{s}_e, \theta(e)) \underline G_e(\underline a_e, \underline r_e) \in \mathbb{R}^{n_A},
\end{equation*}
where $\underline{\underline{\vartheta}}(\underline s_e,s_0) = [\underline{\underline{\vartheta}}(s_1,s_0), \underline{\underline{\vartheta}}(s_2,s_0),...,\underline{\underline{\vartheta}}(s_{n_{B}},s_0)] \in \mathbb{R}^{n_A \times n_An_{B}}$ is the neural tangent kernel evaluated for all $ \underline s_e, \; s_0$ pairs. 
\end{theorem}
\begin{proof}
First, we concatenate the states and returns $(\underline s_e, s_0)$, $(\underline G_e(\underline a_e, \underline r_e), \underline G_0(a_0, r_0))$, respectively and solve Eq.~\eqref{eq:pi_nonlin_multi} for $s_0$.
For more insight, we illustrate the matrix multiplication in Eq.~\eqref{eq:pi_nonlin_multi} with the concatenated state in Figure \ref{fig:aug_matrix_any}.
\begin{figure}[htb!]
\centering
\includegraphics[scale=0.5]{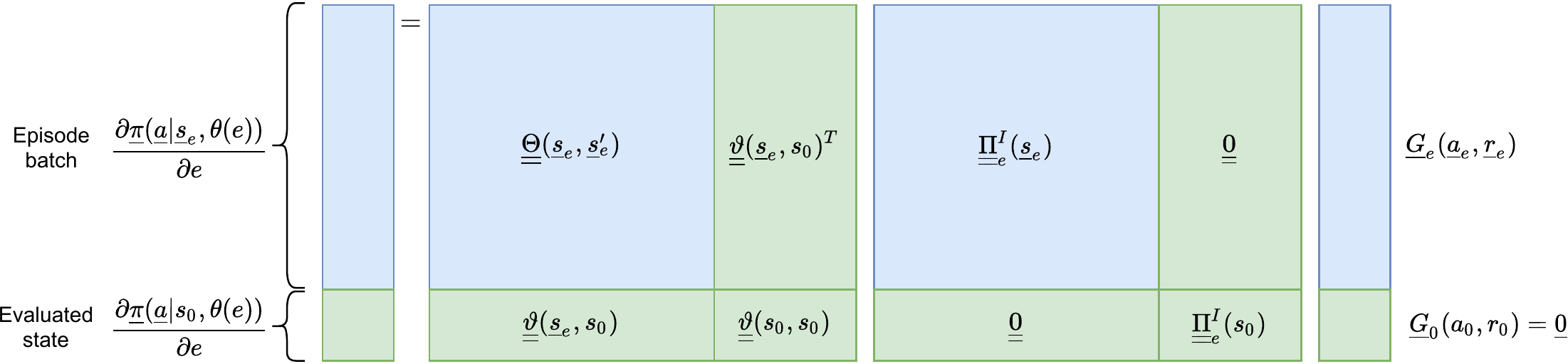}
\caption{Matrix compatibility with the augmented batch. The blue section indicates the original Eq.~\eqref{eq:pi_nonlin_multi}, while the green blocks stem from the augmentation.}
\label{fig:aug_matrix_any}
\end{figure}
\begin{align}
\nonumber
&    \frac{\partial\underline \pi(a \mid s_0, \theta(e))}{\partial e} =  \underline{\underline{\vartheta}}(\underline s_e,s_0) \underline{\underline{\Pi}}^I_e(\underline{s}_e, \theta(e)) \underline G_e(\underline a_e, \underline r_e)  + \\ & + \underline{\underline{\vartheta}}( s_0,s_0) \underline{\underline{\Pi}}^I_e({s}_0, \theta(e)) \underline G_0(\underline a_0, \underline r_0) .
\end{align}
The NTK is based on the partial derivatives of the policy network and can be evaluated anywhere. Therefore, $\underline{\underline{\vartheta}}(\underline s_e,s_0) = [\underline{\underline{\vartheta}}(s_1,s_0), \underline{\underline{\vartheta}}(s_2,s_0),...,\underline{\underline{\vartheta}}(s_{n_{B}},s_0)] \in \mathbb{R}^{n_A \times n_An_{B}}$ can be computed for any state. $\underline{\underline{\vartheta}}(\underline s_e,s_0)$ consists of symmetric $n_A \times n_A$ blocks. 
Since $s_0$ is not included in the learning, it does not affect the policy change, its return is zero for every action, $\underline G_0(a_0, r_0) = \underline 0 \in \mathbb{R}^{n_A}$. The zero return cancels out the term $\underline{\underline{\vartheta}}( s_0,s_0) \underline{\underline{\Pi}}^I_e({s}_0, \theta(e)) \underline G_0( a_0, \underline r_0) \in \mathbb{R}^{n_A}$. Therefore, 
\begin{equation}
\label{eq:policy_eval}
\frac{\partial\underline \pi(a \mid s_0, \theta(e))}{\partial e} =  \underline{\underline{\vartheta}}(\underline s_e,s_0) \underline{\underline{\Pi}}^I_e(\underline{s}_e, \theta(e)) \underline G_e(\underline a_e, \underline r_e) \in \mathbb{R}^{n_A}.
\end{equation}
\end{proof}

\section{The NTK-based constrained REINFORCE algorithm}
\label{sec:constr}
Every physical system has some limits (saturation). Intuitively, the agent should not increase the control action towards the limit if the system is already in saturation (for example, if a robot manipulator is at an end position, increasing the force in the direction of that end position is pointless). 
Assuming we have some idea how the agent should behave (which actions to take) at specific states $\underline s_s = [s_{s1}, s_{s2}, ..., s_{sn_S}]$ or regions $\underline f_s(t)$, equality and inequality constraints can be prescribed for the policy. Define equality and inequality constraints as $\underline{\pi}_{ref,eq}(\underline a_s\mid \underline s_s) = \underline c_{eq}$ and $\underline{\pi}_{ref,ineq}(\underline a_s\mid \underline s_s) \geq \underline c_{ineq}$, where $[\underline c_{eq}, \; \underline c_{ineq}] \in \mathbb{R}^{n_S}$ is a vector constant probabilities. Additionally, suppose that one constrained region in the state space is described with $f_s(t): \mathbb{R} \rightarrow \mathcal{S}$. Then, in the constrained regions we assume to have a reference policy ${\underline \pi}_{ref,reg}(\underline a_s\mid \underline f_s(t)) \geq c_{reg}$, $\forall t \in [t_{min}, \; t_{max}]$.

In the sequel, relying on Theorem \ref{theo:learningdynamics_anywhere}, we provide the mathematical deduction on how to enforce constraints during learning.

\subsection{Equality constraints}
\label{sec:eq_constr}
To deal with constraints, we augment Eq.~\eqref{eq:pi_nonlin_multi} with the constrained state-action pairs. Visually, the policy change at the augmented batch states are shown in Figure \ref{fig:aug_matrix}. Since the desired policy change at the safe states can be given as
\begin{equation}
    \Delta \underline \pi(\underline a_s \mid \underline s_s, \theta(e)) = \underline{\pi}_{ref,eq}(\underline a_s\mid \underline s_s) - \underline{\pi}(\underline a_s\mid \underline s_s, \theta(e)),
\end{equation}
$\Delta \underline \pi(\underline a_s \mid \underline s_s, \theta(e)) \in \mathbb{R}^{n_S}$. The only unknowns are the returns $\underline G_s = [G_{s1}, G_{s2}, ..., G_{sn_S}]$ for the safe actions at the safe states. Note that the upper block of Figure \ref{fig:aug_matrix} contains differential equations, while the lower block consists of algebraic equations. It is sufficient to solve the algebraic part. With the lower blocks of Figure \ref{fig:aug_matrix} we can write the linear equation system
\begin{align}
\label{eq:constrained_system}
\nonumber 
& \Delta\underline \pi(\underline a_s \mid \underline s_s, \theta(e)) -  \underline{\underline{\vartheta}}(\underline s_e, \underline s_s)\underline{\underline{\Pi}}^I_e(\underline{s}_e, \theta(e)) \underline G_e(\underline a_e, \underline r_e) = \\ 
&  \underline{\underline{\vartheta}}(\underline s_s, \underline s_s')\underline{\underline{\Pi}}^I_e(\underline{s}_s, \theta(e))\underline G_s.
\end{align}
This system has a single unique solution as there are $n_S$ unknown returns and $n_S$ equations and it can be solved with e.g.,~the DGSEV algorithm \cite{haidar2018harnessing}.
\begin{figure}[htb!]
\centering
\includegraphics[scale=0.5]{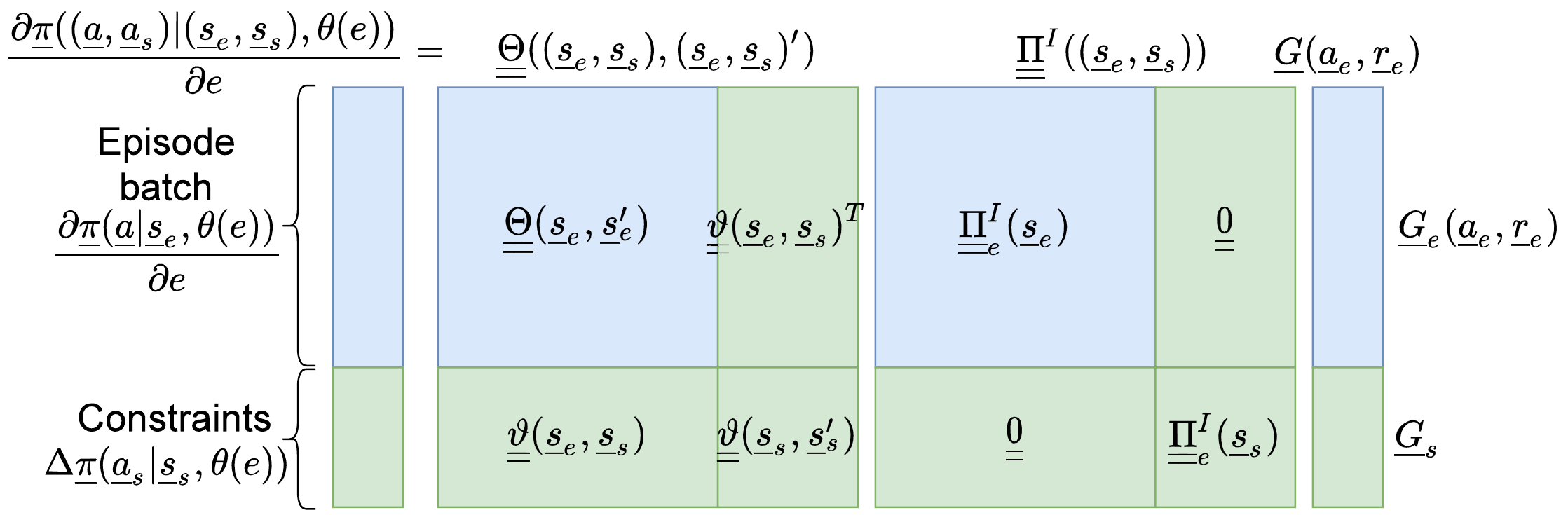}
\caption{Matrix compatibility of the constrained policy change evaluation (concatenated batch)}
\label{fig:aug_matrix}
\end{figure}

In order to obey the constraints, a safe data batch is constructed. We concatenate the safe states, actions, and computed returns with the episode batch as: $\{(\underline s_e, \underline s_s), (\underline a_e, \underline a_s), (\underline G_e(\underline a_e, \underline r_e), \underline G_s)\}$. Then, the agent's weights are updated with the appended batch with gradient ascent, Eq.~\eqref{eq:gradientascent}.

In the initial stages of learning, the difference between the reference policies and the actual ones will be large. Therefore high rewards are needed to eliminate this difference. This also means that the effect of the collected batch on the weights is minor compared to the safe state-action-return tuples. In addition, large returns might cause loss of stability during the learning. The returns computed from the linearized policy change might differ from the actual one, especially if large steps are taken (i.e.,~large returns are applied). On the other hand, when the policy obeys the constraints, the computed returns will be small and will only compensate for the effect of the actual batch. In a special case when the return is zero, the result is the unconstrained policy change at the specific state as in Section \ref{sec:anypoint}.
In addition, if the policy is smooth, the action probabilities near the constrained points will be similar. In continuous state space, this implies that defining grid-based (finite) constraints is sufficient. 

\begin{remark}
\textbf{Time complexity:} The critical operations are kernel evaluations and solving the linear equation system. The DGSEV algorithm used for solving the linear equation system has time complexity  $\mathcal{O}(n_S^3)$ \cite{haidar2018harnessing}. The time complexity of kernel evaluations is $\mathcal{O}((n_{B} + n_S)n_S)$. If the kernel is computed for every output channel at the batch states, time complexity increases to $\mathcal{O}((n_An_{B} + n_S)n_S)$.
\end{remark}

\subsection{Inequality constraints}
\label{sec:ineq_constr}
In the same way, inequality constraints can be prescribed too. Assume the there are some states of the environment where an action shall be taken with at least a constant probability: $\underline{\pi}_{ref,ineq}(\underline a_s\mid \underline s_s) \geq \underline c_{ineq}$. Then, similar to Eq.~\eqref{eq:constrained_system}, the inequality constraints can be written as
\begin{align}
\label{eq:ineq_constraints}
    &\Delta \underline \pi(\underline a_s \mid \underline s_s, \theta(e)) -  \underline{\underline{\vartheta}}(\underline s_e, \underline s_s)\underline{\underline{\Pi}}^I_e(\underline{s}_e, \theta(e)) \underline G_e(\underline a_e, \underline r_e) \\ \nonumber &  \leq  \underline{\underline{\vartheta}}(\underline s_s, \underline s_s')\underline{\underline{\Pi}}^I_e(\underline{s}_s, \theta(e))\underline G_s.
\end{align}
Solving this system of inequalities can be turned into a convex quadratic programming problem. Since the original goal of reinforcement learning is learning on the collected episode batch data, the influence of the constraints on the learning (i.e.,~the magnitude of $\underline G_s$) should be as small as possible. Therefore, the quadratic program is
\begin{equation}
    \min_{G_{si}} \sum_{i = 1}^{n_S}G_{si}^2
\end{equation}
\begin{center}
        subject to Eq.~\eqref{eq:ineq_constraints}.
\end{center}
Note that, the quadratic cost function is needed to similarly penalize positive and negative returns. Quadratic programming with interior point methods has polynomial time complexity ($\mathcal{O}(n_S^3)$, \cite{ye1989extension}) and has to be solved after every episode. Alternatively, it is possible to relax the constraints via introducing slack variables \cite{boyd2004convex}. This approach turns the optimization with hard constraints into soft ones, which can be solved even if there are conflicting constraints.
In practical applications, such numerical errors are more more of an issue than time complexity (i.e.~solving the optimization with thousands of constraints).

\subsection{Constrained regions - dynamically changing constraints}
\label{sec:regional_constraints}
The equality and inequality constraints we proposed in the previous subsections have some limitations, especially in high-dimensional state spaces. This subsection, explores whether the proposed methodology can be extended to constrained regions of the state space. Note that handling entire regions is computationally very intensive; therefore, we focus our efforts on dynamically changing (within a region) point-wise constraints. Prescribing point-wise constraints is not only computationally easier but also less stringent than constraining whole regions. The prerequisite for satisfying a constraint in a whole region require powerful function approximators, i.e., neural networks with many tuneable parameters. 

In order to handle regional constraints we modify Eq.~\eqref{eq:ineq_constraints} to admit $f_s(t)$:
\begin{align}
\label{eq:region_constraints}
    &\Delta \pi(a_s \mid f(t), \theta(e)) -  \underline{\underline{\vartheta}}(\underline s_e, f_s(t))\underline{\underline{\Pi}}^I_e(\underline{s}_e, \theta(e)) \underline G_e(\underline a_e, \underline r_e) \\ \nonumber &  \leq  \underline{\underline{\vartheta}}(f_s(t), f_s(t'))\underline{\underline{\Pi}}^I_e(f_s(t), \theta(e)) G_s(t).
\end{align}
Note that $G_s(t)$ is a scalar function in this case. I.e.,~we seek for a return function that forces the policy to obey the prescribed constraint in the region characterized by $f_s(t)$. This yields a non-convex optimization to find $G_s(t)$ that satisfies Eq.~\eqref{eq:region_constraints} $\forall t$. 

Under some assumptions, this optimization can be simplified. Assuming $\underline{\underline{\vartheta}}(f_s(t), f_s(t'))\underline{\underline{\Pi}}^I_e(f_s(t), \theta(e))$ is positive and invertible $\forall t$, one can write 
\begin{align}
\label{eq:region_constraints2}
    &\left(\Delta \pi(a_s \mid f_s(t), \theta(e)) -  \underline{\underline{\vartheta}}(\underline s_e, f_s(t))\underline{\underline{\Pi}}^I_e(\underline{s}_e, \theta(e)) \underline G_e(\underline a_e, \underline r_e)\right) \\ \nonumber &   \left(\underline{\underline{\vartheta}}(f_s(t), f_s(t'))\underline{\underline{\Pi}}^I_e(f_s(t), \theta(e)) \right)^{-1} \leq  G_s(t).
\end{align}
In other words, $G_s(t)$ is lower-bounded by a nonlinear function. 
Since we are looking for a single return value instead of a function we pick a single $G_{so} \geq G_s(t_o)$, $t_o \in [t_{min}, \; t_{max}]$. Here we present two approaches to do so. 
\subsubsection{Maximum return}
First, we select $t_o$ where $G_s(t_o)$ has its maximum within $[t_{min}, \; t_{max}]$.
For a continuous $t$, it can be done by computing the derivative of the left-hand side of Eq.~\eqref{eq:region_constraints2}:
\begin{align}
    &\frac{d}{dx}\left[ \left(\Delta \pi(a_s \mid f_s(t), \theta(e)) -  \underline{\underline{\vartheta}}(\underline s_e, f_s(t))\underline{\underline{\Pi}}^I_e(\underline{s}_e, \theta(e)) \underline G_e(\underline a_e, \underline r_e)\right)\right. \\ \nonumber &   \left.\left(\underline{\underline{\vartheta}}(f_s(t), f_s(t'))\underline{\underline{\Pi}}^I_e(f_s(t), \theta(e)) \right)^{-1} \right] = 0,
\end{align}
and evaluating the original function at the roots of the above equation and evaluating the original function at these points plus at $t_{min}$ and $t_{max}$. Then we can select $G_s(t_o)$ that is greater or equal to the local extremum. 

The numerical solution to the above is much easier. First, we discretize $t$ with $t_i \in [t_{min}, \; t_{max}]$ and evaluate Eq.~\eqref{eq:region_constraints2} at every $t_i$.
\begin{align}
\label{eq:constraint_selection_maxreturn}
    & t_o = \underset{t_i}{argmax} \left( \left(\Delta \pi( a_s \mid f_s(t), \theta(e)) -  \underline{\underline{\vartheta}}(\underline s_e, f_s(t_i))\underline{\underline{\Pi}}^I_e(\underline{s}_e, \theta(e)) \underline G_e(\underline a_e, \underline r_e)\right)\right. \nonumber \\  &   \left.\left(\underline{\underline{\vartheta}}(f_s(t_i), f_s(t_i))\underline{\underline{\Pi}}^I_e(f_s(t_i), \theta(e)) \right)^{-1}\right).
\end{align}

Then, we select $G_{so} \geq G_s(t_o)$. We use only a single $G_{so}$ corresponding to the environment state $f_s(t_o)$ during the episode-by-episode learning. The extremum $G_s(t_o)$ will be at different $t$ values. Therefore, if the function approximator is capable, eventually, the constraint will be satisfied on the whole region. Intuitively, it would be possible to use multiple returns per region, to speed up episode-wise learning at the cost of higher computational demand. 

The main drawback of this approach is its computational demand. Despite we use only one constraint per region, we have to evaluate Eq.~\eqref{eq:region_constraints2} for every $t_i$, which involves evaluating the kernel too, which has $\mathcal{O}(n_B + 1)$ time complexity ($n_S = 1$). Therefore, with $n_t$ discrete steps of $t$, we have $\mathcal{O}(n_B + 1)n_t$.

\subsubsection{Maximum policy deviation}
It is possible to take a simplified approach. Instead of solving Eq.~\eqref{eq:region_constraints2} to find the return that has the most effect in the region, we can look for the point, where the policy deviates the most from the reference (worst-case). 
\begin{equation}
    t_o = \underset{t}{argmax}\left( {\pi}_{ref,reg}( a_s\mid f_s(t)) - {\pi}(a_s\mid f_s(t), \theta(e))\right), 
\end{equation}
or in the discrete case
\begin{equation}
\label{eq:constraint_selection_maxdiff}
    t_o = \underset{t_i}{argmax} \left({\pi}_{ref,reg}( a_s\mid f_s(t_i)) - {\pi}( a_s\mid f_s(t_i), \theta(e))\right). 
\end{equation}
Then, we solve Eq.~\eqref{eq:ineq_constraints} for the state $f_s(t_o)$.

Since in these two methods we always select one constrained state per region, these states change episode-by-episode, we could label them dynamic constraint selection too. 
The dynamical regional constraining methods are further discussed in \ref{app:reg_constr} through numerical simulations. 

\begin{remark}
\textbf{Multiple constraints per region.}
The simplest but least efficient way to tackle regional constraints would be to use every discrete $t_i$ within the region resulting in $n_t$ constraints per region. For large state spaces and fine discretizations, solving the optimization might be numerically challenging. The above two approaches can be extended to multiple returns too. We could select multiple states per region that would affect the learning the most. Additionally, we could further discretize each region.
\end{remark}

\begin{remark}
\textbf{Other policy gradient methods.}
One extension of REINFORCE is policy gradient with baseline. There, a baseline (typically the value function) is subtracted from the returns to reduce variance. The policy is then updated with these modified returns using the policy gradient theorem \cite{sutton2018reinforcement}. Constraints can be adapted to the policy gradient with baseline too. Since the returns at the constrained states are shaped to satisfy specific action probabilities, baselines should not be subtracted from the safe returns. Therefore, in a constrained REINFORCE with baseline, batch returns are offset by the baseline while the safe returns are not. 
\end{remark}


We summarize the NTK-based constrained REINFORCE algorithm in Algorithm \ref{alg:REINFORCE_constr}.
\begin{algorithm}[htb!]
\caption{The NTK-based constrained REINFORCE algorithm}
\begin{algorithmic}[1]
\State Initialize $e = 1$.
\State Define equality constraints $\underline{\pi}_{ref,eq}(\underline a_s\mid \underline s_s) = \underline c_{eq}$.
\State Define inequality constraints $\underline{\pi}_{ref,ineq}(\underline a_s\mid \underline s_s) \geq \underline c_{ineq}$.
\State Define regional constraints 
${\underline \pi}_{ref,reg}(\underline a_s\mid \underline f_s(t)) \geq c_{reg}$.
\State Initialize the policy network with random $\theta$ weights. 
\While{not converged} 
\For{every regional constraint $(f_{s1}(t),\; f_{s2}(t),\;...)$}
\State Find the dynamically constrained state $f_s(t_o)$ with
\State "Maximum return" (Eq.~\eqref{eq:constraint_selection_maxreturn}) 
\State \textbf{or}
\State "Maximum policy deviation" (Eq.~\eqref{eq:constraint_selection_maxdiff}) 
\State Append $f_s(t_o)$ to the set of constrained states.
\EndFor
\State Generate a Monte-Carlo trajectory $\{\underline{s}_e(k), \underline{a}_e(k), \underline{r}_e(k)\}$, 

$k = 1,2, ..., n_{B}$ with the current policy $\pi(\underline{a}_e(k) \mid \underline{s}_e(k), \theta(e))$. 
\For{the whole MC trajectory ($k=1,2,...,n_{B}$)}
\State Compute the returns $G_e(k)$ with Eq.~\eqref{eq:return}.
\EndFor
\State Construct $\underline{\underline{\vartheta}}(\underline s_e, \underline s_s)$, $\underline{\underline{\vartheta}}(\underline s_s, \underline s_s')$, $ \underline{\underline{\Pi}}^I_e(\underline{s}_e, \theta(e))$, $ \underline{\underline{\Pi}}^I_e(\underline{s}_s, \theta(e))$.
\State Compute $\underline G_s$ with Eq.~\eqref{eq:constrained_system} \textbf{and}  Eq.~\eqref{eq:ineq_constraints}.
\State Concatenate the MC batch and the constraints:

$\{(\underline s_e, \underline s_s), (\underline a_e, \underline a_s), (\underline G_e(\underline a_e, \underline r_e), \underline G_s)\}$.
\For{the augmented MC trajectory ($k=1,2,...,n_{B}+n_S$)}
\State Update policy parameters with gradient ascent (Eq.~\eqref{eq:gradientascent}). \EndFor
\State Increment $e$.
\EndWhile
\end{algorithmic}
\label{alg:REINFORCE_constr}
\end{algorithm}

\section{Experimental studies}
\label{sec:benchmarks}
We investigate the proposed constrainted learning algorithm in two OpenAI gym environments with increasing complexity: Cartpole \cite{barto1983neuronlike} and Lunar lander (Figure \ref{fig:gym}), \cite{brockman2016openai}. The simplicity of the Cartpole environment enables us to perform in depth analysis of the proposed algorithm. The lunar lander demonstrates its efficiency in higher-dimensional environment. 
\begin{figure}[htb!]
\centering
  \centering
  \includegraphics[width=1\linewidth]{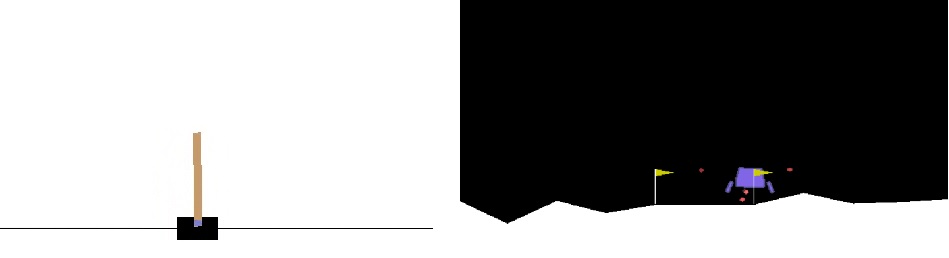}
  \caption{Visualization of the environments. Left: Cartpole, right: Lunar lander.}
\label{fig:gym}
\end{figure}

The learning agent is a 2 layer deep fully-connected ReLU network with softmax output nonlinearity and appropriate input-output sizes (i.e.,~4 inputs and 2 outputs for Cartpole, and  8 inputs and 6 outputs for Lunar lander). The hidden layer width is $5000$ neurons with bias terms. That is to comply with the assumptions in \cite{jacot2018neural}, i.e.,~a shallow and wide neural network. On the other hand, the NTK can be computed for more complex NN structures too e.g.,~\cite{yang2019fine}. Note that the primary purpose is not finding the best function approximator, merely demonstrating the efficiency of the constrained learning. To achieve lazy learning, the learning rate is set to $\alpha = 0.0001$. The small learning rate ensures that the approximation of the policy change remains accurate, see \ref{app:accuracy}. 

\subsection{Cartpole}
\label{sec:cartpole}
The cart pole problem (also known as the inverted pendulum) is a common benchmark in control theory as it can be easily modeled as a linear time invariant system \cite{skogestad2007multivariable}. The goal is to balance a pole to remain upright by horizontally moving the cart. The agent in this environment can take two actions: accelerating the cart left ($a^0$) or right ($a^1$). The cart pole has four states: the position of the cart ($x$), its velocity ($\dot x$), the pole angle ($\varphi$), and the pole angular velocity ($\dot \varphi$). In the reinforcement learning setting, the agent's goal is to balance the pole as long as possible. Reward is given for every discrete step if the pole is in vertical direction, and the episode ends if the pole falls or successfully balancing for 200 steps. The pass criteria for this gym environment is reaching an average reward of 195 for 100 episodes.

Constraints are implemented through intuition. If the pole is tilted too much  right, the cart must move right to balance it, and vice versa. Selecting too many states to constrain slow down the computation significantly while defining conflicting constraints can make Eq.~\eqref{eq:ineq_constraints} unsolvable. Inequality constraints are imposed on the pole angle and its angular velocity at discrete cart positions resulting in 18  constrained states, see Table \ref{tab:constraints} in \ref{app:constraints}. 

With the proper selection of constraints, it is possible to train the agent in 5 steps. Figure \ref{fig:CartPoleTrained} shows the learning of the constrained agent. The agent learns the problem within a few episodes but it can be attributed to the simplicity of the problem and the choice of constraints. For this environment the constraints are chosen based on what the agent should do intuitively. Thus, once the constraints are fulfilled the agent will be able to balance the pole. The agent can only learn new policies at states which are not constrained, thus the final policy will be sub-optimal (if the constraints are sub-optimal). This is a trade-off between more sample efficient and safe learning and optimality. The episode loss is presented in Figure \ref{fig:CartPoleLoss}. It is initially large because the constraints are not yet fulfilled (elements in $\underline G_s$ are large). Figure \ref{fig:CartPolePolicy} depicts the section of the learned policy with the constrained states. The resulting policy is very smooth, and the constraints are respected.
\begin{figure}[htb!]
\centering
  \centering
  \includegraphics[width=0.8\linewidth]{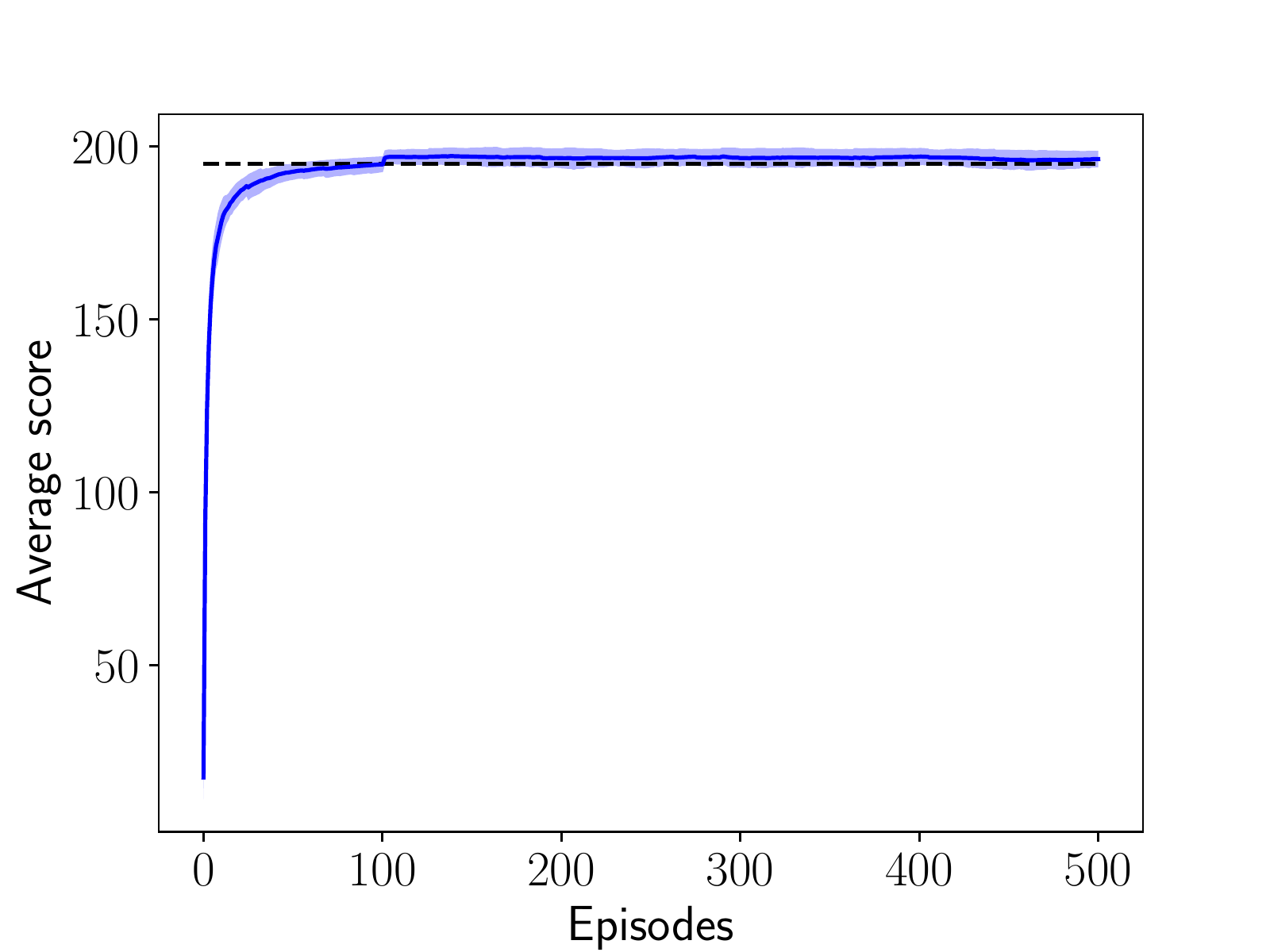}
  \caption{Solving Cartpole with constrained-PG for 10 different random seeds. Pass criteria is an average score above 195.}
\label{fig:CartPoleTrained}
\end{figure}
\begin{figure}[htb!]
\centering
  \centering
  \includegraphics[width=0.8\linewidth]{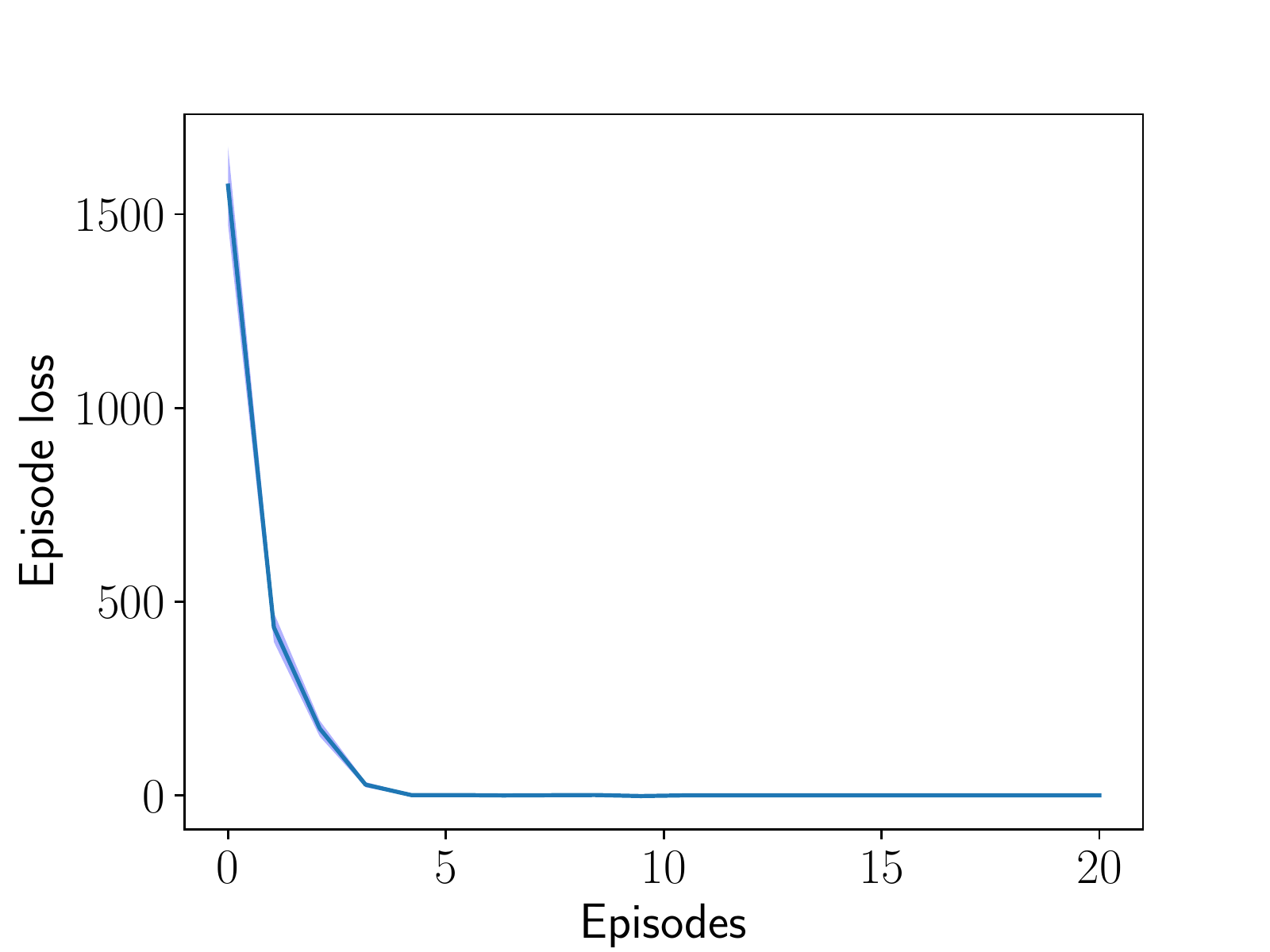}
  \caption{Episode loss (Cartpole)}
\label{fig:CartPoleLoss}
\end{figure}
\begin{figure}[htb!]
\centering
  \centering
  \includegraphics[width=0.8\linewidth]{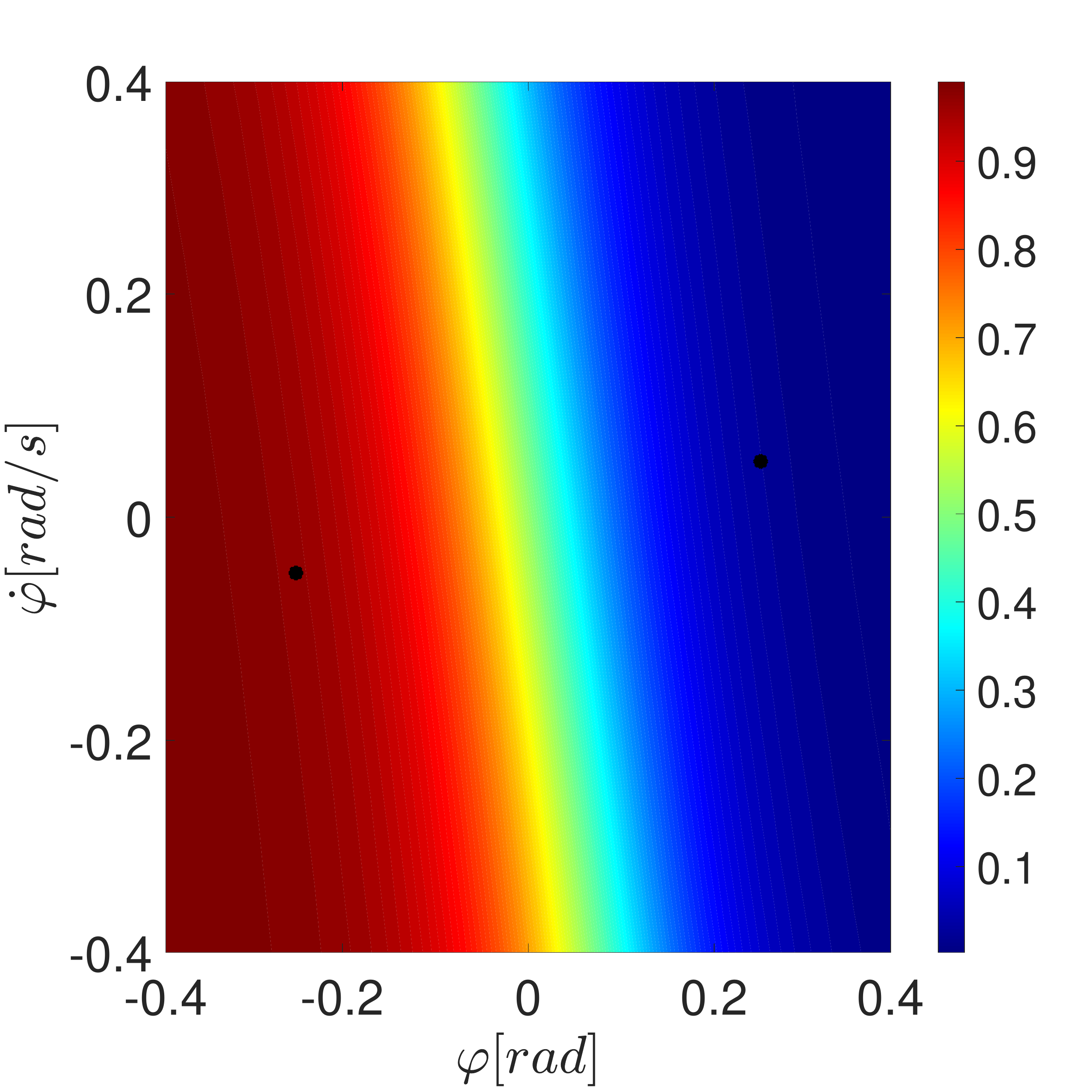}
  \caption{Probability of taking $a^0$ (push the cart left) in the Cartpole environment at $x=0$, $\dot x = 0$. Dots represents the constrained states $s_{s5}$, $s_{s14}$ (Table \ref{tab:constraints})}
\label{fig:CartPolePolicy}
\end{figure}

\FloatBarrier

Next, instead of point-wise constraints we introduce constrained regions inspired by Figure \ref{fig:CartPolePolicy}. We define two main regions bounded by 
\begin{align}
    f_1(t): & \hspace{0.3cm} \dot \varphi \leq -8 \varphi -2, \\ f_2(t): & \hspace{0.3cm} \dot \varphi \geq -8 \varphi +1.2.
\end{align}
Each region is further discretized into three sub-regions (Table \ref{tab:constraints3}). Then, each sub-region is discretized into 50 states. We followed the "Maximum policy deviation" constraint selection strategy for every region, yielding 6 dynamically constrained states.
Simulation results are on par with the statically constrained simulations, see Figure \ref{fig:CartPoleTrained_reg}. The agent with the regional constraints learn slightly slower. This is possibly due to the more stringent constraints: there are less constrained states but their location is changing from episode to episode. This is reflected in the evolution of the loss too (Figure \ref{fig:CartPoleLoss_reg}). It takes more episodes for the loss to disappear and it starts from a higher initial value due to the different selection of constraints. Finally, Figure \ref{fig:CartPolePolicy_reg} shows the learned policy, which is almost identical to the statically constrained one. 
\begin{figure}[htb!]
\centering
  \centering
  \includegraphics[width=0.8\linewidth]{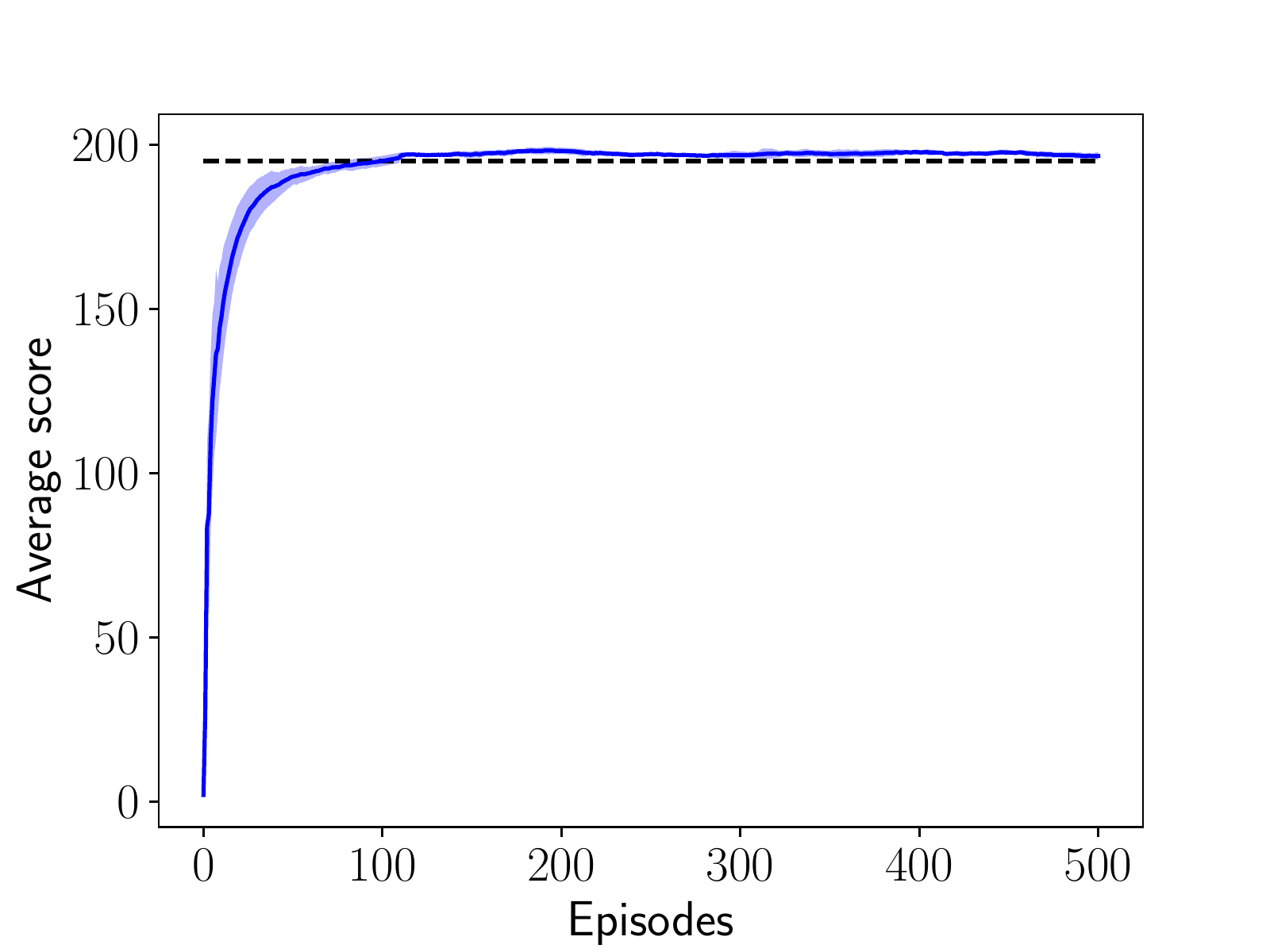}
  \caption{Solving Cartpole with regionally constrained-PG for 10 different random seeds. }
\label{fig:CartPoleTrained_reg}
\end{figure}
\begin{figure}[htb!]
\centering
  \centering
  \includegraphics[width=0.8\linewidth]{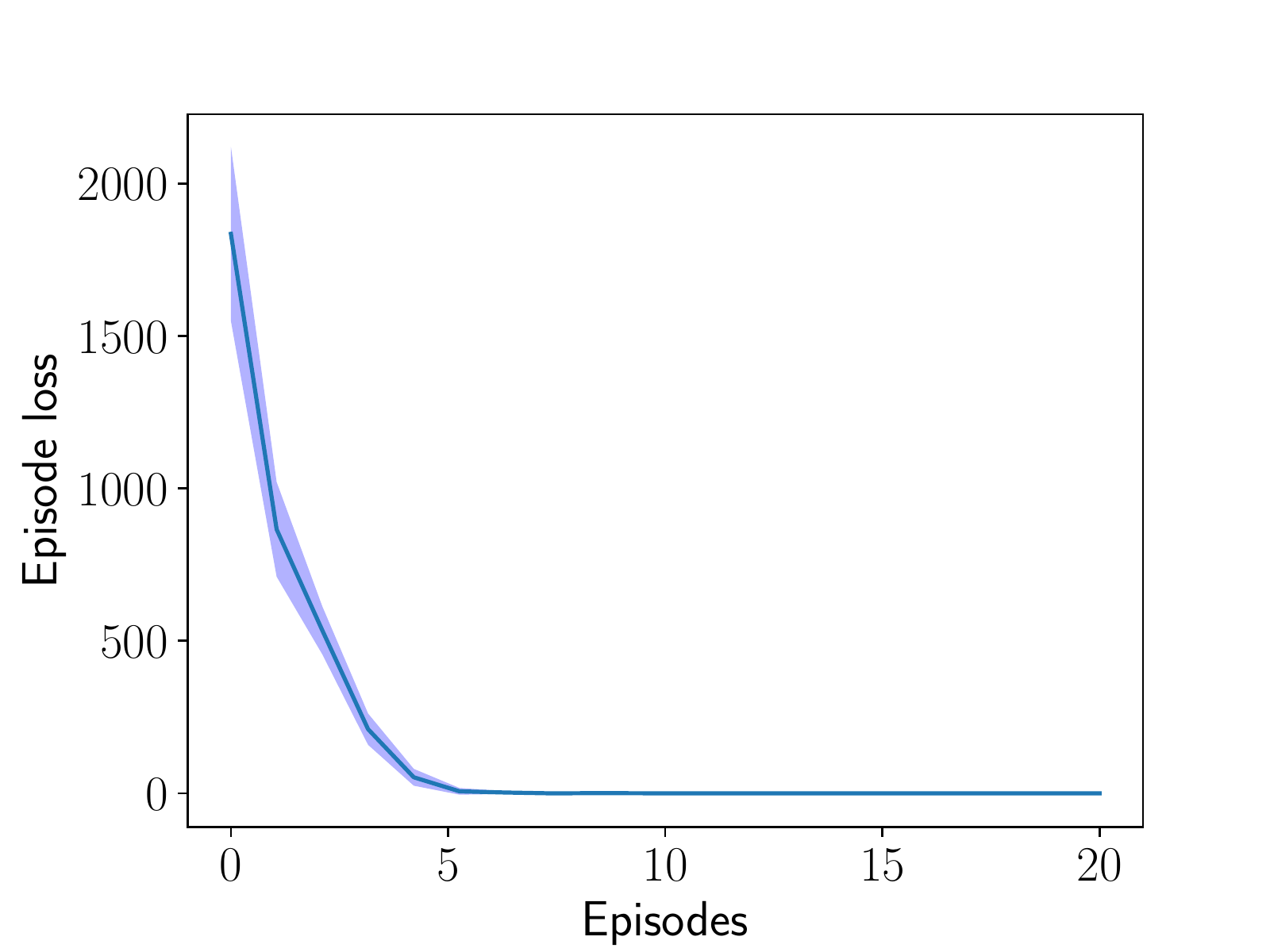}
  \caption{Episode loss (Cartpole, regional constraints)}
\label{fig:CartPoleLoss_reg}
\end{figure}
\begin{figure}[htb!]
\centering
  \centering
  \includegraphics[width=0.8\linewidth]{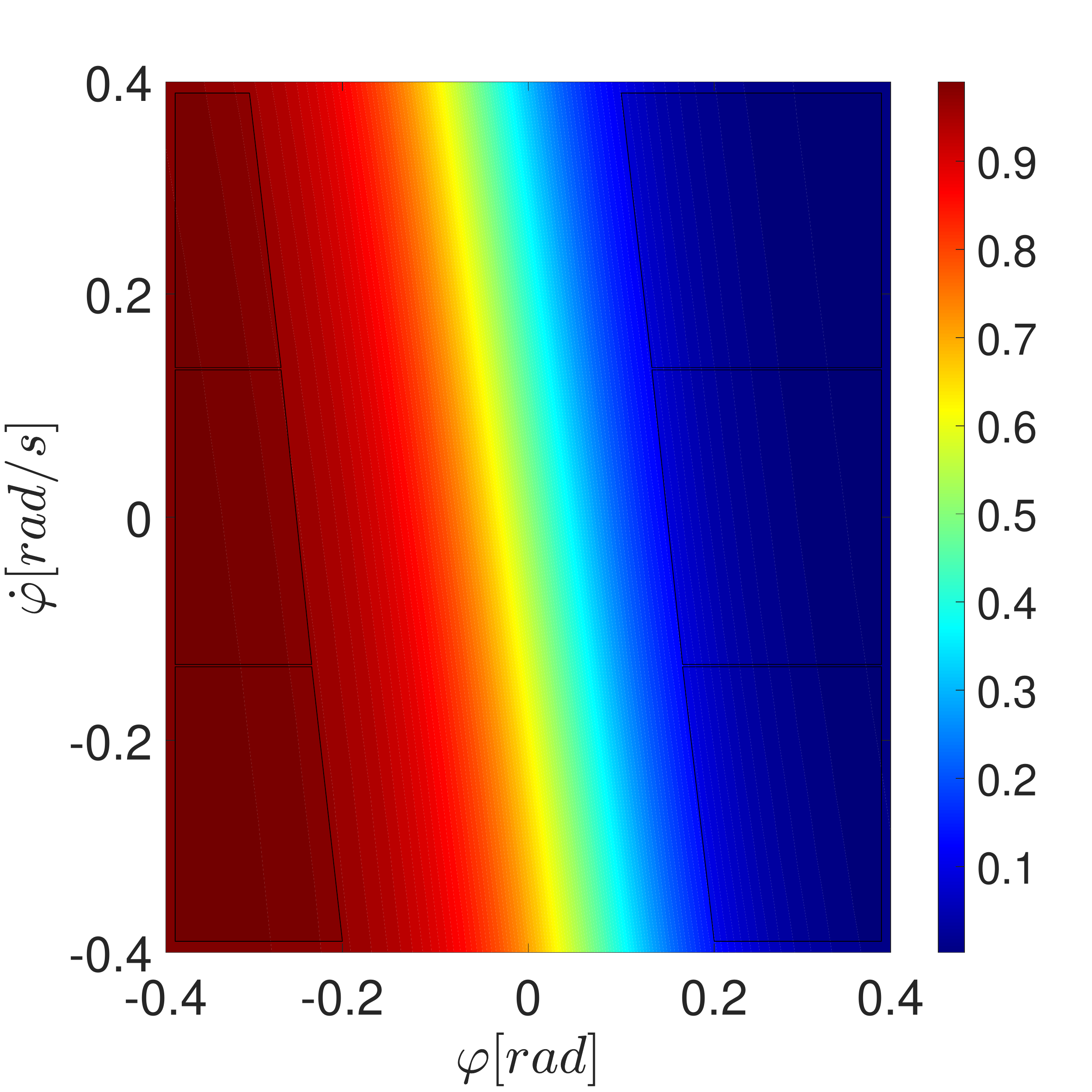}
  \caption{Probability of taking $a^0$ (push the cart left) in the Cartpole environment at $x=0$, $\dot x = 0$. Gray regions represent the constraints described in  Table \ref{tab:constraints3}}
\label{fig:CartPolePolicy_reg}
\end{figure}

These results would place the algorithm in the top 10 of the OpenAI gym leaderboard for this environment, competing with deterministic policies and closed-loop controllers, and being much faster than traditional RL methods (e.g.,~\cite{kumar2020balancing, gymLeaderboard}). Figure \ref{fig:Cartpole_benchmark} shows how the agent learns with the commonly used (unconstrained) double deep Q learning \cite{van2016deep}. The constraining makes learning much faster as it eliminates the need to explore states that are known to be unsafe.
\begin{figure}[htb!]
\centering
  \centering
  \includegraphics[width=0.8\linewidth]{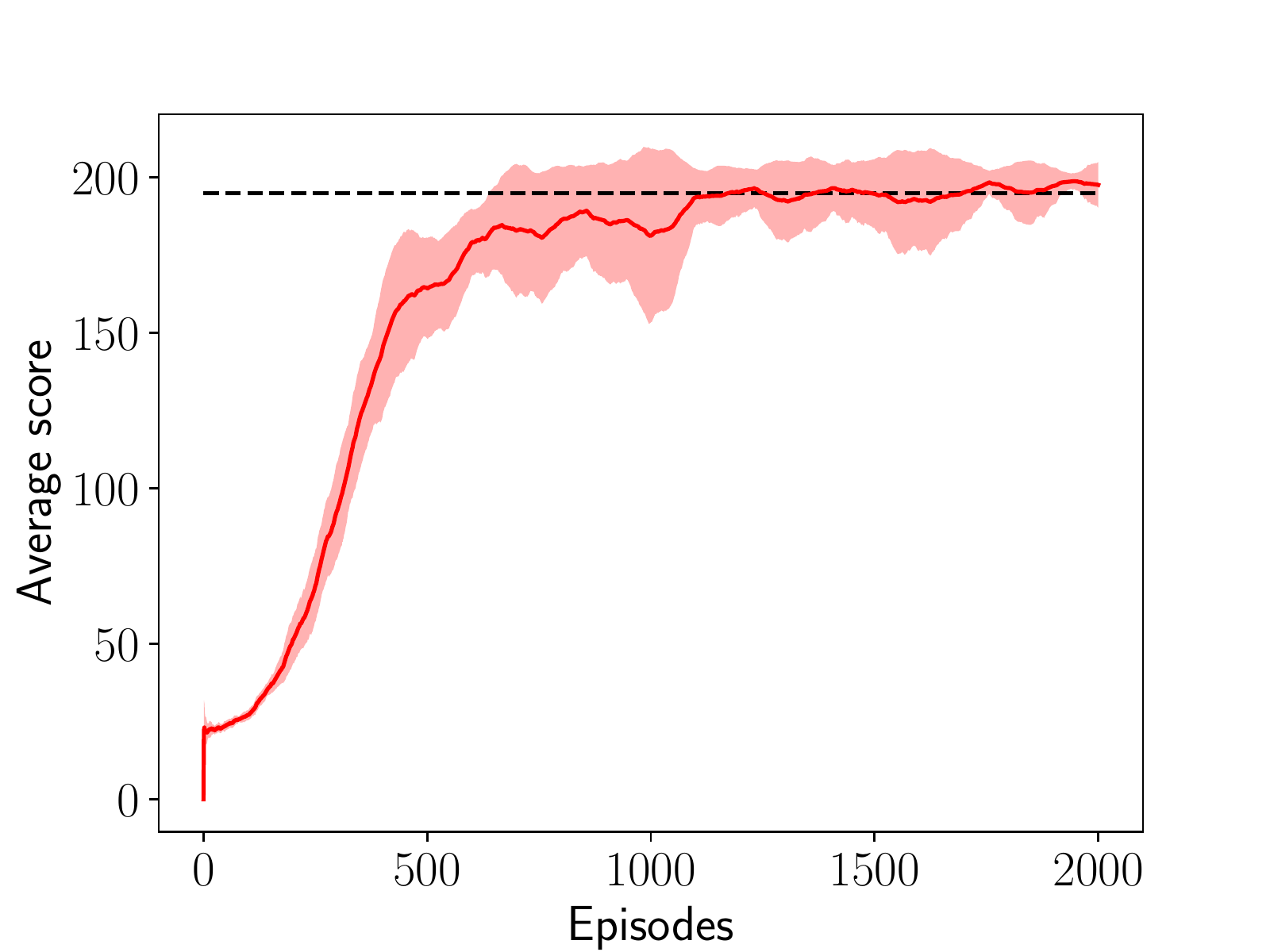}
  \caption{Solving Cartpole with double deep Q learning for 10 different random seeds.}
\label{fig:Cartpole_benchmark}
\end{figure}

\subsection{Lunar lander}
\label{sec:lunarlander}
The goal in this 2D environment is to land a rocket on a landing pad without crashing. The agent in this environment has eight states: its horizontal and vertical coordinates ($x$, $y$) and velocities ($\dot x$, $\dot y$), its angle $\varphi$, its angular velocity $\dot \varphi$ and the logical states whether the left and right legs are in contact with the ground ($l_{left}$ and $l_{right}$). The agent can choose from four actions: 0: do nothing, 1: fire the left thruster, 2: fire the main engine, and 3: fire the right thruster. The episode finishes if the lander crashes or
comes to rest. Reward is given for successfully landing close to the landing pad. Crashing the rocket results in a penalty. Firing the engines (burning fuel) also results in small penalties. In OpenAI gym the pass criteria for solving this environment is an average reward of 200 for 100 episodes.

During unconstrained learning, the two most common reasons observed for episode failures were the lander crashing too fast into the ground and tilting over mid-flight. To this end, constraints were imposed on the vertical velocity and the angle of the lander. Based on these empirical observations we propose inequality constraints. 

Constraints are imposed to keep the lander on an ideal trajectory: as the lander comes closer to the ground, it should decelerate by firing the main engine (simulating hover slam). If the rocket has a too large horizontal velocity or is tilted, the side engines should be used. The proposed constraints are summarized in Table \ref{tab:constraints2} (\ref{app:constraints}). 
In order to ease computational load, only every $10^{th}$ step of the Monte-Carlo trajectory was logged. Too long trajectories slow down computation as the NTK has to be evaluated at more points, which has polynomial time complexity. 

With the above setup, the agent was able to land after a few episodes successfully. However, after 500 episodes of training, the 100 runs average reward was just below 200, see Figure \ref{fig:LunarLanderTrained}. The loss was initially high but since the constraints enforced a "good" policy, it declined quickly. 
\begin{figure}[htb!]
\centering
  \centering
  \includegraphics[width=0.8\linewidth]{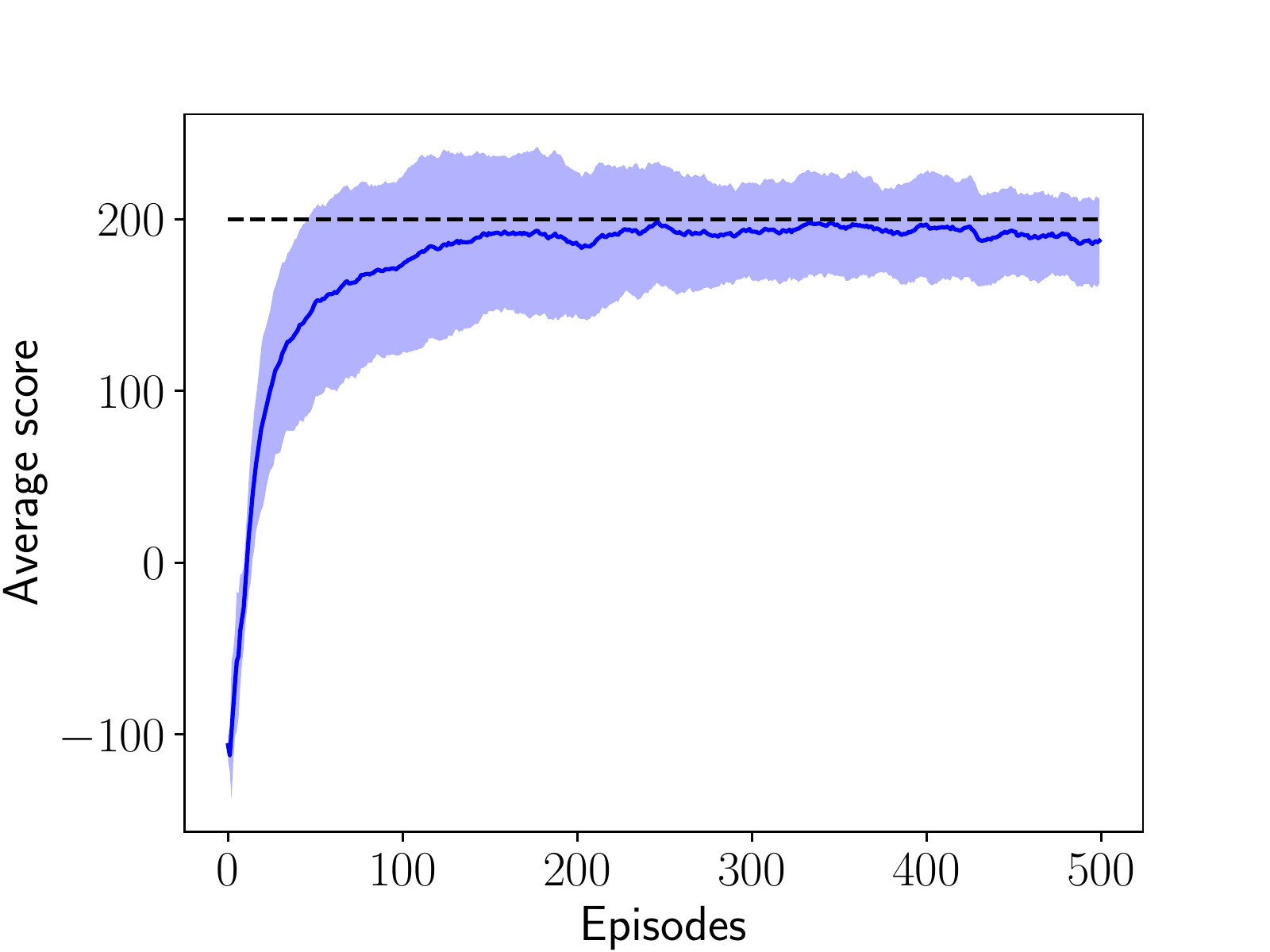}
  \caption{Solving Lunar lander with constrained-PG for 10 different random seeds.  Pass criteria is an average score above 200.}
\label{fig:LunarLanderTrained}
\end{figure}
\begin{figure}[htb!]
\centering
  \centering
  \includegraphics[width=0.8\linewidth]{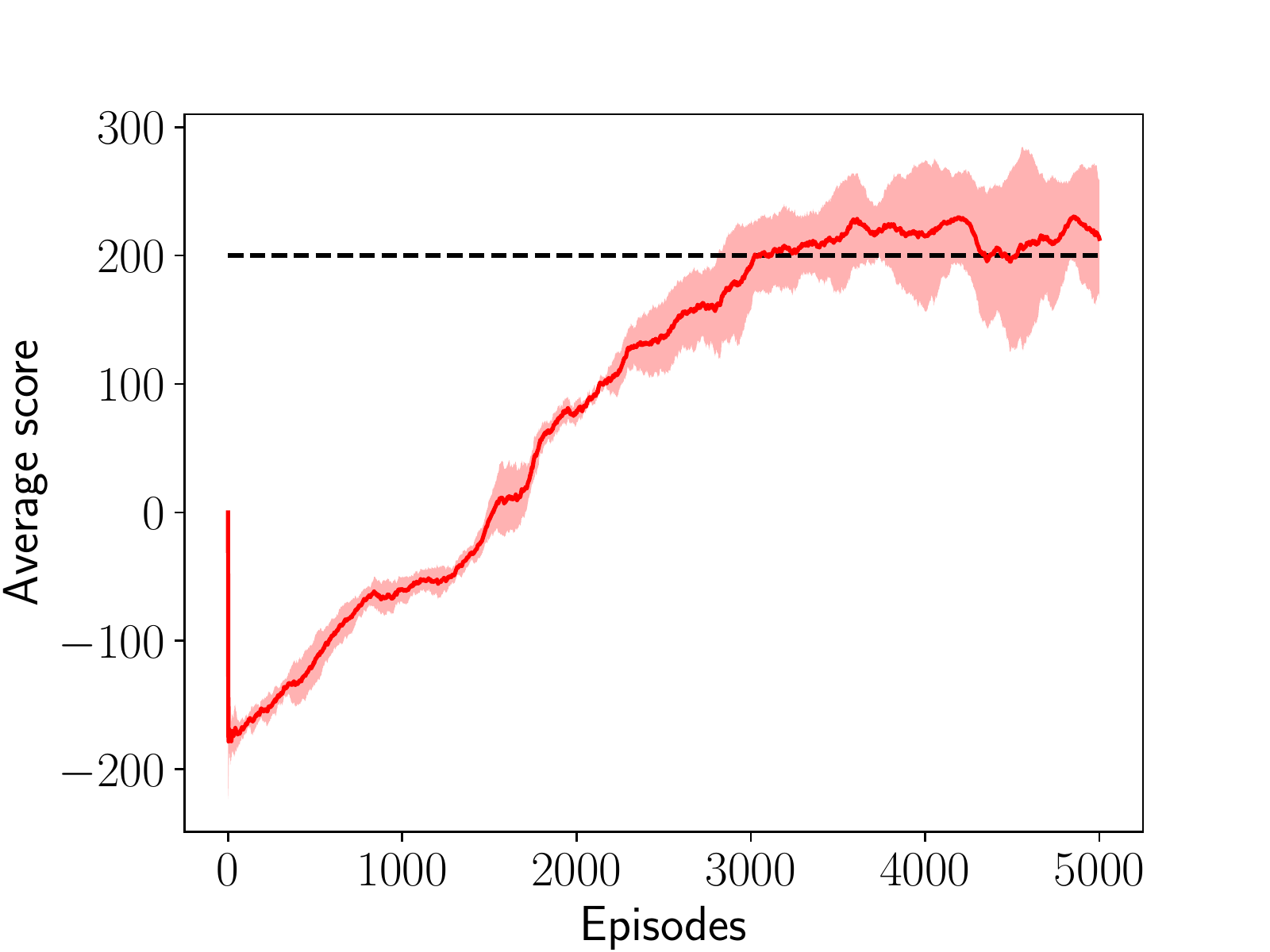}
  \caption{Solving Lunar lander with double deep Q learning for 10 different random seeds}
\label{fig:LunarLander_benchmark}
\end{figure}
\begin{figure}[htb!]
\centering
  \centering
  \includegraphics[width=0.8\linewidth]{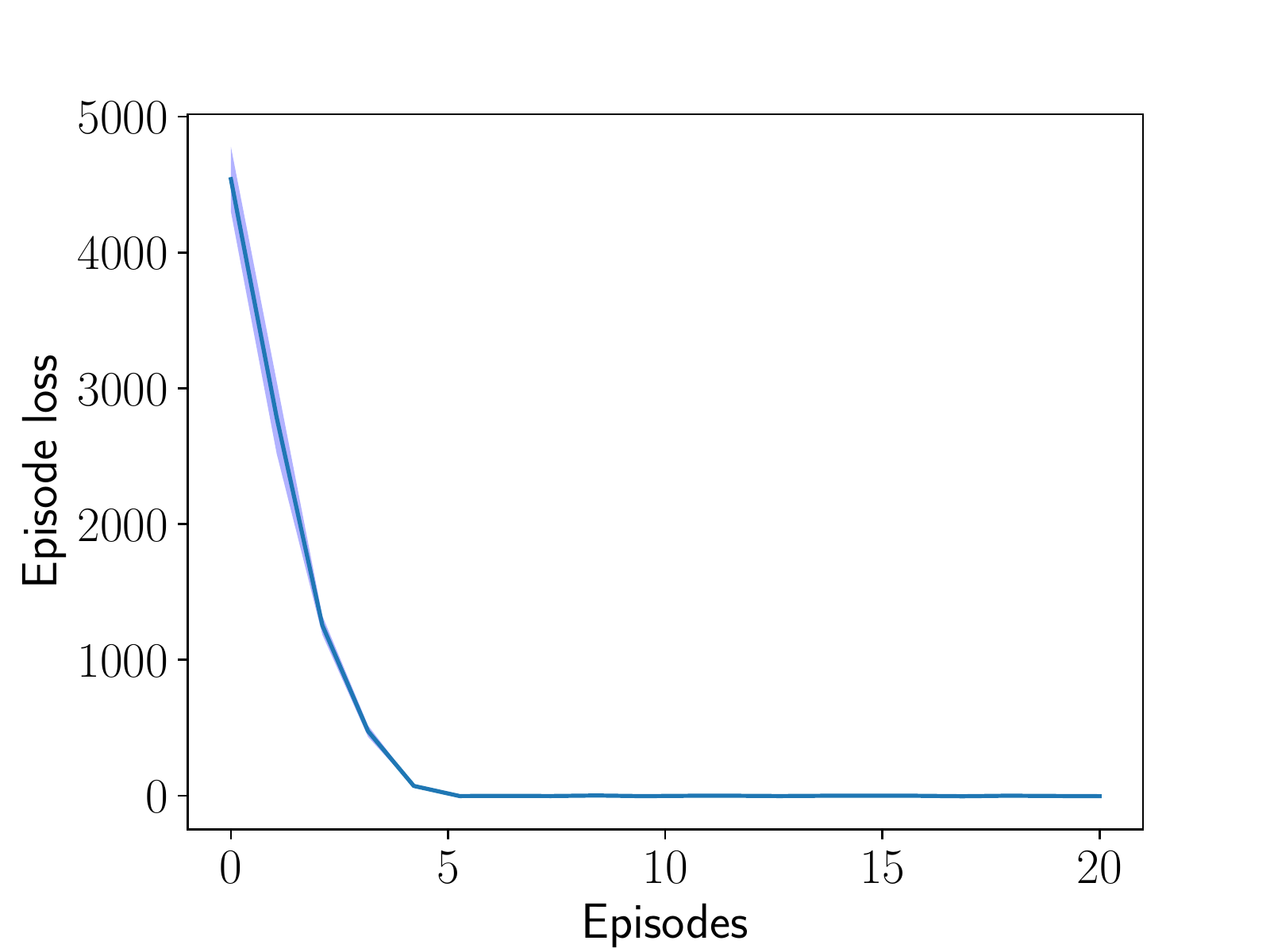}
  \caption{Episode loss (Lunar lander)}
\label{fig:LunarLanderLoss}
\end{figure}

The results in this environment shed light on some important features of the proposed algorithm. The agent can within a few episodes find a good policy if the constraints are set up right. On the other hand, poor selection of constraints can harm the performance on the long run. The algorithm can achieve similar performance in significantly less steps compared to other benchmarks e.g.,~\cite{van2016deep, gadgil2020solving}. 
Figure \ref{fig:LunarLander_benchmark} shows how the agent learns with double deep Q learning \cite{van2016deep}. The unconstrained method learns one order of magnitude slower, however it can reach slightly higher average scores by the end of the training. Therefore, there is a trade-off between speeding up learning via constraining and reaching the highest possible average score. On the other hand, enforcing constraints during learning reduces oscillations and the variance from different random seeds, that is a common issue for typical benchmarks. Therefore, it is easier to know when to stop learning. 
With more careful selection of constraint states, the final average score of the agent could be improved, i.e.,~the constraints would not hamper reaching the optimal policy. This highlights one more drawback of the constrained approach: if the dimension of the environment space is large, the number of required constraints in a grid-based fashion increase significantly (i.e.,~the curse of dimensionality applies). Therefore, a lot more manual tuning effort is required to achieve safe and fast learning.

\section{Conclusions}
\label{sec:concl}
We proposed a solution to augment the REINFORCE algorithm with equality, inequality, and dynamically changing regional constraints. To this end, the policy evolution was computed with the help of the neural tangent kernel. Then, arbitrary states were selected with desired action probabilities. Next, for these arbitrary state-action pairs returns were computed that approximately satisfy the prescribed constraints under gradient ascent.
The efficiency of the constrained learning was demonstrated with a shallow and wide ReLU network in the Cartpole and Lunar Lander OpenAI gym environments. Results suggest that constraints are satisfied after 2-3 episodes. If they are set up correctly, learning becomes extremely fast (episode-wise) while satisfying safety constraints, thus ensuring some transparency of the policy too. If the constraints are satisfied the returns for safety become small, only slightly influencing learning from the Monte-Carlo trajectory episode batch. 
On the other hand, selecting suitable constraints requires expert knowledge about the environment. Therefore, the proposed algorithm is best suited for controlled physical systems where saturations and unsafe states can be pinpointed, and countermeasures can be explicitly defined. On the other hand, the learning algorithm may suffer from the curse of dimensionality: in high-dimensional state-spaces setting up constraints manually is tedious. In addition, solving the optimization for several constraints has polynomial time complexity. 

As a future line of research, other variants of policy gradient methods will be analyzed. We hypothesize that more complex policy-based approaches can be augmented with constraints using the NTK too. Furthermore, we intend to improve our take on constrained regions. We intend to solve the constrained optimization on a different domain. I.e., transform regional constraints to a domain where they simplify to linear constraints. Note that the proposed "Maximum policy deviation" approach does exactly this in a simplified manner. The transformation is simply the $max$ function. 

\section*{Declarations and statements}
\subsection*{Funding}
This work has been supported and funded by the project RITE (funded by CHAIR, Chalmers University of Technology).
\subsection*{Statement of interest}
The authors have no relevant financial or non-financial interests to disclose.
\subsection*{Ethics approval}
Not applicable.
\subsection*{Consent} 
All authors consent to participating in the paper and publishing their individual data or image. 
\subsection*{Author contribution}
All authors contributed to the study conception and design. Material preparation, implementation and analysis were performed by Balázs Varga. The first draft of the manuscript was written by Balázs Varga and all authors commented on previous versions of the manuscript. All authors read and approved the final manuscript.
\subsection*{Data availability}
Not applicable.
\subsection*{Code availability}
Source codes will be made publicly available.

\pagebreak

\begin{appendices}

\section{Accuracy of the prediction}
\label{app:accuracy}
Due to the linearization in the NTK (gradient calculation), by neglecting the non-differentiable property of the ReLU activation at the origin, as well as by numerical precision errors the NTK-based policy change prediction might be biased. In this appendix this bias is analyzed by comparing the predicted policy update (for the whole batch, Eq.~\eqref{eq:batch_change}) $\frac{\partial\underline \pi^B(\underline a \mid \underline s_e, \theta(e))}{\partial e}$ with the actual one (at the batch average state) $\bar q_e(\underline s_e, \underline a_e, \theta(e), \theta(e+1)) = \frac{1}{n_{B}}\sum_{k=1}^{n_{B}}(\underline \pi(a_e(k) \mid s_e(k),\theta(e+1)) -  \underline \pi(a_e(k) \mid s_e(k),\theta(e))$, assuming data batch $\{ \underline s_e, \underline a_e, \underline r_e\}$.
Assuming no constraints, the agent from Section \ref{sec:cartpole} tries to learn to balance the pole in the cart pole environment for 100 episodes. 
The episode-by-episode relative errors during learning are computed with Eq.~\eqref{eq:rel_error}. Results are shown in Figure \ref{fig:NTK_accuracy} for three different learning rates.
\begin{equation}
\label{eq:rel_error}
    \varepsilon = \frac{\frac{\partial\underline \pi^B(\underline a \mid \underline s_e, \theta(e))}{\partial e} - \bar q_e(\underline s_e, \underline a_e, \theta(e), \theta(e+1)))}{(\bar q_e(\underline s_e, \underline a_e, \theta(e), \theta(e+1))}\cdot 100.
\end{equation}
Results suggest that the prediction becomes more accurate as the learning rate decreases. The learning rate used in the simulations (Section \ref{sec:benchmarks}, $\alpha = 0.001$) yields approximately $0.05\%$ prediction error.  
\begin{figure}[htb!]
\centering
\begin{subfigure}{.55\textwidth}
  \centering
  \includegraphics[trim={0 0 0 1cm},clip,width=1\linewidth]{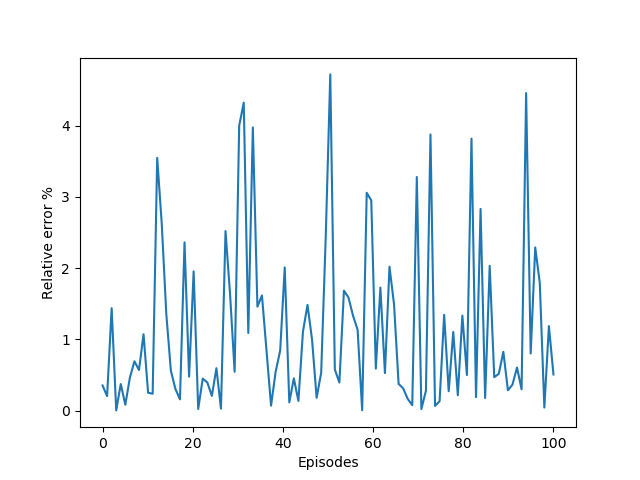}
  \caption{$\alpha = 0.001$}
\end{subfigure}%

\begin{subfigure}{.55\textwidth}
  \centering
  \includegraphics[trim={0 0 0 1cm},clip,width=1\linewidth]{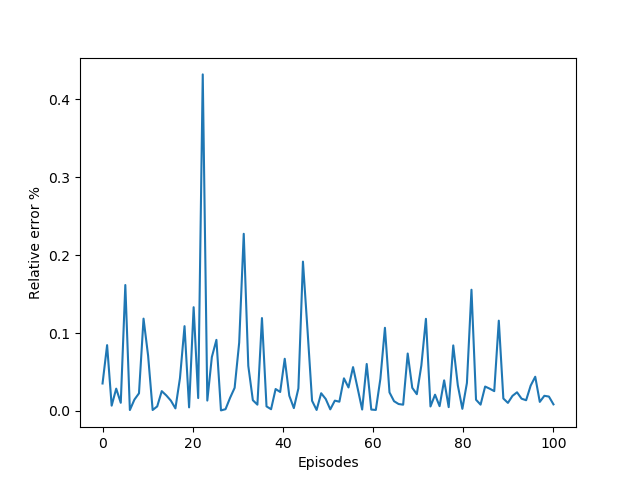}
  \caption{$\alpha = 0.0001$}
\end{subfigure}

\begin{subfigure}{.55\textwidth}
  \centering
  \includegraphics[trim={0 0 0 1cm},clip,width=1\linewidth]{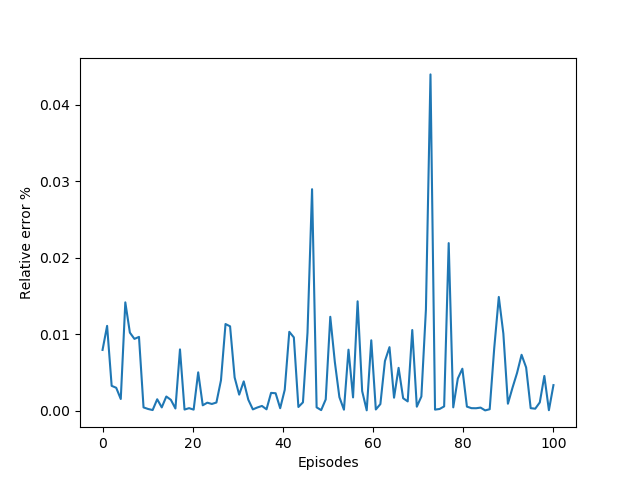}
  \caption{$\alpha = 0.00001$}
\end{subfigure}
\caption{Relative estimation errors between the computed NTK based policy change and the actual, SGD based policy change (Eq.~\eqref{eq:rel_error}). Results are shown for the first output channel of the policy network.}
\label{fig:NTK_accuracy}
\end{figure}

\FloatBarrier

\section{Simulation-based analysis of the regional constraints}
\label{app:reg_constr}
Learning with the regional constraints proposed in Section \ref{sec:regional_constraints} is further investigated using numerical simulations. To this end, the Cartpole environment is used with the agent introduced in Section \ref{sec:benchmarks}. For visualization purposes, we only consider the last two states of the environment: the pole angle $\varphi$, and its angular velocity $\dot \varphi$ at the $x=0$, $\dot x = 0$ slice of the state space. 
First, we introduce four circular 2D disks as regional constraints, see Table \ref{tab:regionalconstraints}.
\begin{table}[h]
\centering
\begin{tabular}{l|cc|c} \hline
          & Center    & Radius (on $\varphi$-$\dot \varphi$) & ${\pi}_{ref,reg}(a^0\mid f_{s}(t))$ \\ \hline
$f_{s1}(t)\!$  & $[0,\; 0,\; -0.2 \; -0.2]$    & $0.05$  & $\geq 0.95$  \\
$f_{s2}(t)\!$   & $[0,\; 0,\; -0.2 \; 0.2]$    & $0.05$  & $\leq 0.05$  \\ 
$f_{s3}(t)\! $  & $[0,\; 0,\; 0.2 \; -0.2]$    & $0.05$  & $\leq 0.05$  \\
$f_{s4}(t)\!$    & $[0,\; 0,\; 0.2 \; 0.2]$    & $0.05$  & $\geq 0.95$ \\
\hline                                        
\end{tabular}
\caption{Constrained regions}
\label{tab:regionalconstraints}
\end{table}
Next, we discretize each disk into 30 points ($t_i$). Then, the constrained learning starts according to the "Maximum return" or the "Maximum policy deviation strategy". In every episode, for every discrete $t_o$ the return is selected with Eq.~\eqref{eq:constraint_selection_maxreturn} or Eq.~\eqref{eq:constraint_selection_maxdiff}, and the environment states corresponding to $t_o$ are used to solve Eq.~\ref{eq:ineq_constraints}. Selection of $t_o$ with the two strategies is depicted in Figure \ref{fig:Regional_constraints_X}. For both strategies, in the $0^{th}$ episode, when the policy is initialized randomly (i.e.,~approximately $50\%$ to take action $a^0$ everywhere), the computed return or policy deviation is almost the same in the whole region. Then, episode-by-episode, the gain needed or the policy deviation decreases. The location of the state to be constrained is always around the same value. This is due to the presence of the other three regional constraints. For both cases by the $9^{th}$ episode, the regional constraint is fulfilled (0 computed return, or 0 policy deviation). The sinusoidal nature of $f_{s1}(t)$ stems from the discretization of the disk. 
\begin{figure}[htb!]
\centering
\begin{subfigure}{.48\textwidth}
  \centering
  \includegraphics[trim={0 0 0 1cm},clip,width=1\linewidth]{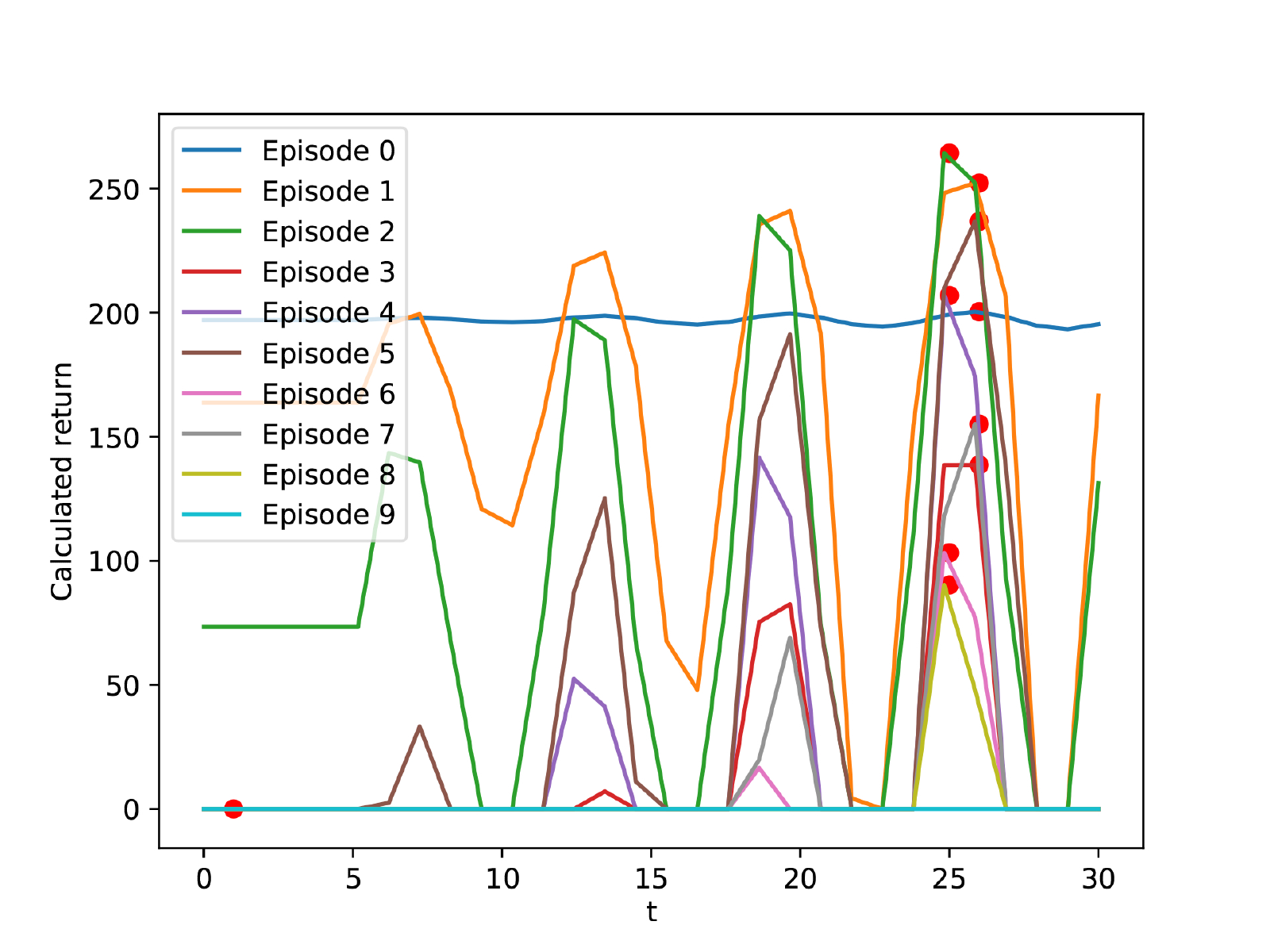}
  \caption{Computed returns in region $f_{s1}(t)$ under the "Maximum return" constraining strategy.}
\end{subfigure}%
\hspace{0.2cm}
\begin{subfigure}{.48\textwidth}
  \centering
  \includegraphics[trim={0 0 0 1cm},clip,width=1\linewidth]{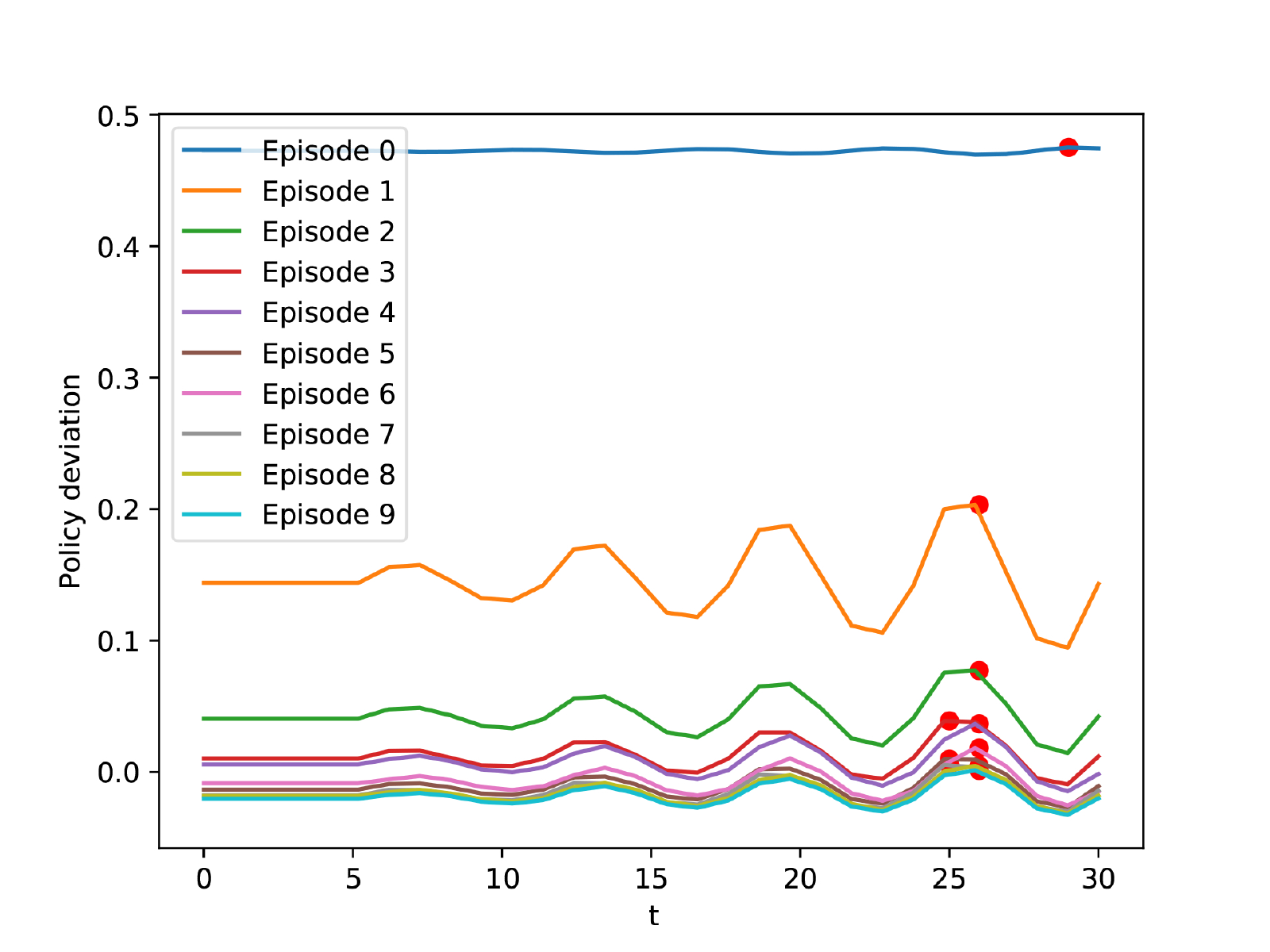}
  \caption{Computed policy deviations in region $f_{s1}(t)$ under the "Maximum policy deviation" constraining strategy.}
\end{subfigure}%

\caption{Functions used in each episode to select the state to constrain}
\label{fig:Regional_constraints_X}
\end{figure}
Figure \ref{fig:Regional_constraints_everything} summarizes the learning with constrained regions in the state space. In the figure, four episodes are shown with the constrained regions (gray areas) and the corresponding constrained states (stars within the areas). In addition, the trajectory the agent traversed in that episode is shown too (black line). In constrained REINFORCE this trajectory and the constrained states are concatenated. Thus, not only the constrained states but the actual MC trajectory contributes to the learning. However, as long as the constraints are not satisfied the rewards from the MC batch is suppressed by the large returns on the constrained states. Therefore, the two constraint selection strategies yield policies that are only slightly different from each other after learning from 9 episodes.  

Note that the prescribed regional constraints in this example are arbitrary and only for demonstration purposes. In the Cartpole environment it yields a bad policy. 

\begin{figure}[htb!]
\centering
  \centering
  \includegraphics[trim={1cm 6cm 0cm 6cm},clip,width=0.66\linewidth]{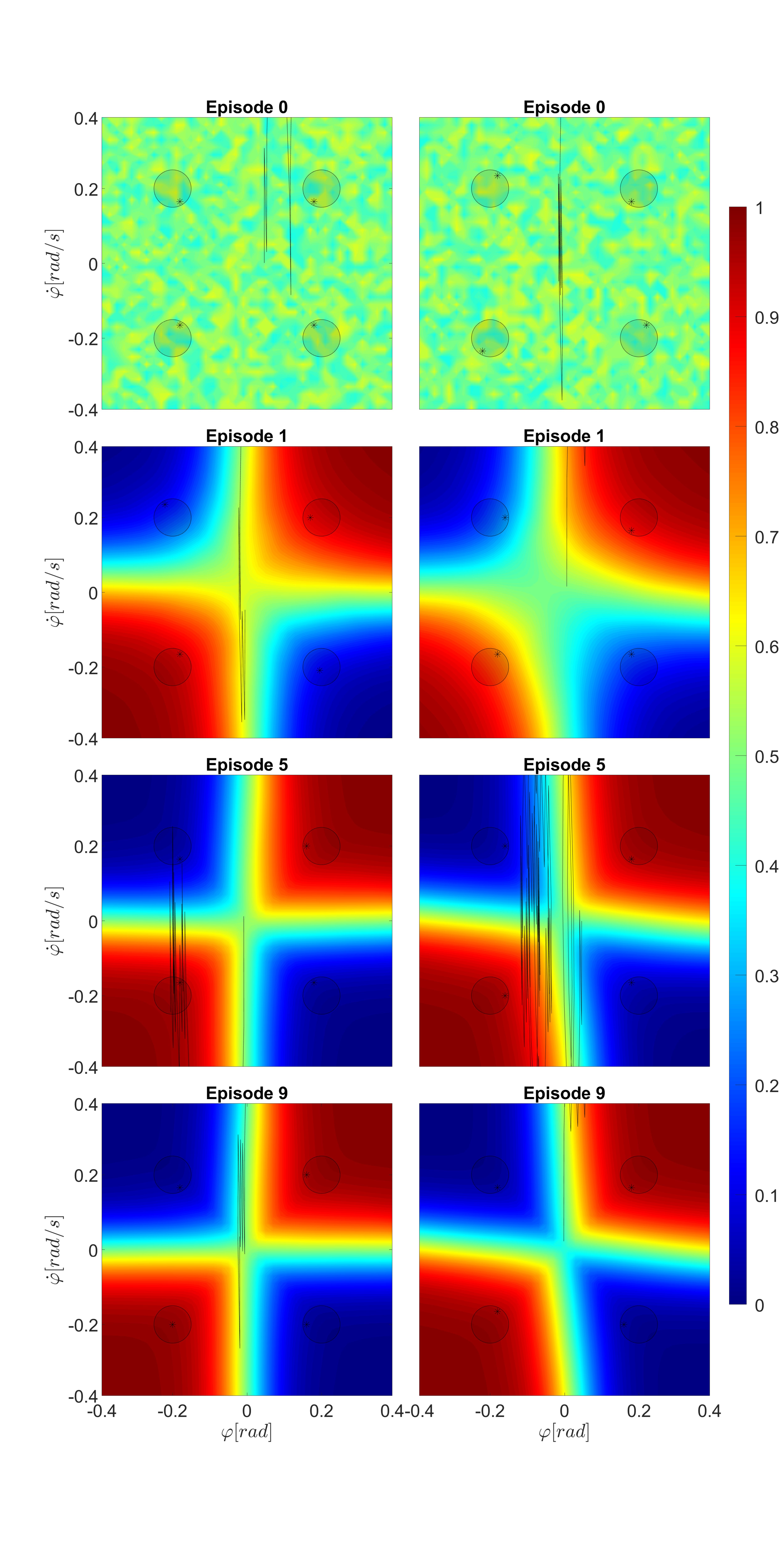}
  \caption{Probability of taking $a^0$ in the Cartpole environment at $x=0$, $\dot x = 0$ under regional constraints $f_{s1}(t),...,f_{s4}(t)$. Constrained regions are denoted with gray circles, and the actual constrained state $f_{s}(t_o)$ within each region by a $\star$. Black lines denote the agent's trajectory in the state-space in the given episode. Plots in the left column show policy evolution with "Maximum return", plots on the right depict the "Maximum policy deviation" constraining strategy.}
\label{fig:Regional_constraints_everything}
\end{figure}

\FloatBarrier

\section{Constraints in the simulations}
\label{app:constraints}
\begin{table}[h]
\centering
\begin{tabular}{l|rrrr|c} \hline
         & $x$    & $\dot{x}$ & $\varphi$ & $\dot{\varphi}$ & $\underline{\pi}_{ref}(a^0\mid s_{si})$  \\ \hline
$s_{s1}$ & $-2$ & $0$       & $0.25$       & $0.05$          &        $\leq 0.05$ \\
$s_{s2}$ & $-1.5$  & $0$       & $0.25$       & $0.05$          &           $\leq 0.05$ \\
$s_{s3}$ & $-1$  & $0$       & $0.25$       & $0.05$          &          $\leq 0.05$ \\
$s_{s4}$ & $-0.5$  & $0$       & $0.25$       & $0.05$          &          $\leq 0.05$ \\
$s_{s5}$ & $0$  & $0$       & $0.25$       & $0.05$          &          $\leq 0.05$ \\
$s_{s6}$ & $0.5$  & $0$       & $0.25$       & $0.05$          &          $\leq 0.05$ \\
$s_{s7}$ & $1$  & $0$       & $0.25$       & $0.05$          &          $\leq 0.05$ \\
$s_{s8}$ & $1.5$  & $0$       & $0.25$       & $0.05$          &          $\leq 0.05$ \\
$s_{s9}$ & $2$  & $0$       & $0.25$       & $0.05$          &          $\leq 0.05$ \\
$s_{s10}$ & $-2$ & $0$       & $-0.25$       & $-0.05$         & $\leq 0.95$        \\
$s_{s11}$ & $-1.5$ & $0$       & $-0.25$       & $-0.05$         & $\geq 0.95$       \\
$s_{s12}$ & $-1$ & $0$       & $-0.25$       & $-0.05$         & $\geq 0.95$         \\
$s_{s13}$ & $-0.5$  & $0$       & $-0.25$       & $-0.05$         & $\geq 0.95$    \\           
$s_{s14}$ & $0$ & $0$       & $-0.25$       & $-0.05$         & $\geq 0.95$         \\
$s_{s15}$ & $0.5$ & $0$       & $-0.25$       & $-0.05$         & $\geq 0.95$         \\
$s_{s16}$ & $1$ & $0$       & $-0.25$       & $-0.05$         & $\geq 0.95$         \\
$s_{s17}$ & $1.5$ & $0$       & $-0.25$       & $-0.05$         & $\geq 0.95$         \\
$s_{s18}$ & $2$ & $0$       & $-0.25$       & $-0.05$         & $\geq 0.95$  \\  \hline   
\end{tabular}
\caption{Constrained states in Cartpole}
\label{tab:constraints}
\end{table}

\begin{table}[h]
\centering
\begin{tabular}{l|lr|c} \hline
             &                           &                     & ${\pi}_{ref,reg}(a^0\mid f_{s}(t))$ \\ \hline
$f_{s1a}(t)$ & $x$ bounds:               & $[-2.5, \; 2.5]$    & $\geq 0.95$                      \\
             & $\dot x$ bounds:          & $0$                 &                                  \\
             & Bounding rectangle      & $(-0.4,\; 0.4 )$    &                                  \\
             & on $\varphi$-$\dot \varphi$: & $(-0.3,\;0.4)$      &                                  \\
             &                           & $(-0.267,\;0.133)$  &                                  \\
             &                           & $(-0.4,\;0.133)$    &                                  \\ \hline
$f_{s1b}(t)$ & $x$ bounds:               & $[-2.5, \; 2.5]$    & $\geq 0.95$                      \\
             & $\dot x$ bounds:          & $0$                 &                                  \\
             & Bounding rectangle      & $(-0.4,\; 0.133 )$  &                                  \\
             & on $\varphi$-$\dot \varphi$: & $(-0.267,\;0.133)$  &                                  \\
             &                           & $(-0.233,\;-0.133)$ &                                  \\
             &                           & $(-0.4,\;-0.133)$   & \multicolumn{1}{l}{}             \\ \hline
$f_{s1c}(t)$ & $x$ bounds:               & $[-2.5, \; 2.5]$    & $\geq 0.95$                      \\
             & $\dot x$ bounds:          & $0$                 &                                  \\
             & Bounding rectangle      & $(-0.4,\; -0.133 )$ &                                  \\
             & on $\varphi$-$\dot \varphi$: & $(-0.233,\;-0.133)$ &                                  \\
             &                           & $(-0.2,\;-0.4)$     &                                  \\
             &                           & $(-0.4,\;-0.4)$     & \multicolumn{1}{l}{}             \\ \hline
$f_{s2a}(t)$ & $x$ bounds:               & $[-2.5, \; 2.5]$    & $\leq 0.05$                      \\
             & $\dot x$ bounds:          & $0$                 &                                  \\
             & Bounding rectangle      & $(0.4,\; -0.4)$     &                                  \\
             & on $\varphi$-$\dot \varphi$: & $(0.2,\;-0.4)$      &                                  \\
             &                           & $(0.167,\;-0.133)$  &                                  \\
             &                           & $(0.4,\;-0.133)$    & \multicolumn{1}{l}{}             \\ \hline
$f_{s2b}(t)$ & $x$ bounds:               & $[-2.5, \; 2.5]$    & $\leq 0.05$                      \\
             & $\dot x$ bounds:          & $0$                 &                                  \\
             & Bounding rectangle      & $(0.4,\; -0.133 )$  &                                  \\
             & on $\varphi$-$\dot \varphi$: & $(0.167,\;-0.133)$  &                                  \\
             &                           & $(0.133,\;0.133)$   &                                  \\
             &                           & $(0.4,\;0.133)$     & \multicolumn{1}{l}{}             \\ \hline
$f_{s2c}(t)$ & $x$ bounds:               & $[-2.5, \; 2.5]$    & $\leq 0.05$                      \\
             & $\dot x$ bounds:          & $0$                 &                                  \\
             & Bounding rectangle      & $(0.4,\; 0.133 )$   &                                  \\
             & on $\varphi$-$\dot \varphi$: & $(0.133,\;0.133)$   &                                  \\
             &                           & $(0.1,\;0.4)$       &                                  \\
             &                           & $(0.4,\;0.4)$       & \multicolumn{1}{l}{}     \\ \hline        
\end{tabular}
\caption{Constrained regions in Cartpole}
\label{tab:constraints3}
\end{table}

\begin{table}[h]
\centering
\begin{tabular}{l|rrrrr|c} \hline
          & $x$    & $y$    & $\dot y$ & $\varphi$ &  $l$  & $\underline{\pi}_{ref}\geq 0.95\!$ \\ \hline
$s_{s1}\!$  & $0$    & $0$  &  $0$           & $0$             & $1$                & $j=0$                                                 \\
$s_{s2}\!$   & $0$    & $0.9$      & $-1$     & $0$              & $0$         & $j=2$                                                 \\
$s_{s3}\! $  & $0$    & $0.5$    & $-0.75$  & $0$             & $0$         & $j=2$                                                 \\
$s_{s4}\!$   & $0$    & $0.2$      & $-0.5$   & $0$             & $0$         & $j=2$                                                 \\
$s_{s5}\!$   & $0$    & $0.1$     & $-0.5$   & $0$             & $0$         & $j=2$                                                 \\
$s_{s6}\!$   & $0$    & $1$        & $0$      & $-0.25$          & $0$         & $j=1$                                                 \\
$s_{s7}\!$   & $0$    & $1$        & $0$      & $0.25$         & $0$         & $j=3$                                                 \\
$s_{s8}\!$   & $0$    & $0.5$      & $0$      & $-0.25$   & $0$             & $j=1$                                                 \\
$s_{s9}\!$   & $0$    & $0.5$       & $0$      & $0.25$          & $0$         & $j=3$                                                 \\
$s_{s10}\!$  & $0.3$  & $1.3$      & $0.1$    & $0$             & $0$         & $j=1$                                                 \\
$s_{s11}\!$  & $-0.3$ & $1.3$       & $-0.1$   & $0$        & $0$         & $j=3$   \\   \hline                                        
\end{tabular}
\caption{Constrained states in Lunar lander}
\label{tab:constraints2}
\end{table}

\end{appendices}

\FloatBarrier

\bibliography{main_wTheorems}


\end{document}